\documentclass[journal,draftclsnofoot,onecolumn]{IEEEtran}
\usepackage{amsmath,amssymb,amsthm}
\usepackage{algorithm}
\usepackage{algpseudocode}
\usepackage{hyperref}
\usepackage{dsfont}
\usepackage{graphicx}
\usepackage{booktabs}
\usepackage{multirow}
\usepackage[table]{xcolor}
\usepackage{stfloats}
\usepackage{float}
\allowdisplaybreaks

\newtheorem{lemma}{Lemma}

\newtheorem{assumption}{Assumption}
\newtheorem{definition}{Definition}
\newtheorem{proposition}{Proposition}

\theoremstyle{remark}

\newcommand{\bX}{\mathbf{X}}
\newcommand{\bt}{\mathbf{t}}

\newcommand{\Ind}{\mathds{1}}

\title{Conformal Changepoint Localization  and Root Cause Analysis with Corrupted Observations}
\author{
Seunghun Yu,
Meiyi Zhu,
Petar Popovski,~\IEEEmembership{Fellow,~IEEE},
Joonhyuk Kang,~\IEEEmembership{Member,~IEEE},\\
and Osvaldo Simeone,~\IEEEmembership{Fellow,~IEEE}

\thanks{This work was partly supported by the Institute of Information \& Communications Technology Planning \& Evaluation (IITP)-ITRC (Information Technology Research Center) grant funded by the Korea government (MSIT) (IITP-2026-RS-2020-II201787, contribution rate: 50\%); in part by the Institute of Information \& Communications Technology Planning \& Evaluation (IITP) under 6G·Cloud Research and Education Open Hub grant funded by the Korea government (MSIT) (IITP-2026-RS-2024-00428780, contribution rate: 50\%).
The work of M. Zhu and O. Simeone was supported by an Open Fellowship of the EPSRC (EP/W024101/1). The work of O. Simeone was also supported by the EPSRC (EP/X011852/1) and the ERC (No. 101198347).
The work of P. Popovski was supported, in part, by the Velux Foundation, Denmark, through the Villum Investigator Grant WATER, nr. 37793.
\\(\textit{Corresponding authors: Joonhyuk Kang and Osvaldo Simeone.})
}

\thanks{Seunghun Yu and Joonhyuk Kang are with the Department of Electrical Engineering, Korea Advanced Institute of Science and Technology, Daejeon 34141, South Korea (e-mail: sh0703.yu@kaist.ac.kr; jhkang@ee.kaist.ac.kr).} 
\thanks{Meiyi Zhu is with the Department of Engineering, King's College London, WC2R 2LS, London, U.K. (e-mail: meiyi.1.zhu@kcl.ac.uk).

Petar Popovski is with the Department of Electronic Systems, Aalborg University, 9220 Aalborg, Denmark (e-mail: petarp@es.aau.dk).

Osvaldo Simeone is with the Institute for Intelligent Networked Systems, Northeastern University London, E1 8PH London, U.K., and also with the Connectivity Section, Department of Electronic Systems, Aalborg
University, 9220 Aalborg, Denmark (e-mail: o.simeone@northeastern.edu).}
}

\begin{document}
% \bstctlcite{BSTcontrol}

\maketitle

\begin{abstract}
   Detecting {when} the statistical behavior of an engineered system changes, and identifying {which} component is responsible, are core problems in the monitoring of telecommunication networks, robotic platforms, security infrastructure, and multi-agent systems. In safety- and mission-critical deployments, such decisions must be accompanied by statistical reliability guarantees rather than by point estimates alone. Conformal changepoint localization (CONCH) and conformal root cause analysis (CROC) meet this need by returning confidence sets that contain the true changepoint, or the true root-cause stream, with a user-specified probability, without parametric assumptions on the data-generating process. In practice, however, observations are frequently corrupted, e.g., by outliers, sensor faults, or adversarial perturbations. While the finite-sample coverage of these procedures is preserved under contamination, the resulting confidence sets can become uninformatively large. Adopting a Huber-type contamination model, this paper proposes weighted CONCH (W-CONCH) and weighted CROC (W-CROC), which downweight observations that are likely to be corrupted with the goal of reducing confidence set size when data may be corrupted. The weighting mechanism, derived from a formal bound on the unknown corrupted data densities, leverages pre-existing second-order classifier-based uncertainty signals, such as those produced by evidential deep learning or Bayesian learning. W-CONCH and W-CROC are further generalized by introducing a meta-learning procedure for the weights that optimizes a differentiable surrogate of the confidence set size. Experiments on image-based and real-world changepoint and root-cause benchmarks show that uncertainty-based weighting substantially reduces confidence set size while maintaining the target coverage.
\end{abstract}


\section{Introduction}
\IEEEPARstart{C}{hangepoint} analysis provides a principled framework for identifying structural changes in ordered data (see Fig. \ref{fig:data_model1})~\cite{truong2020selective,aminikhanghahi2017survey}, while its multi-stream extension, root cause analysis, seeks to attribute an observed change to the component that originated it (see Fig. \ref{fig:data_model2})~\cite{sole2017survey, soldani2022anomaly}.
This paper studies both problems under two key requirements for safety-critical applications: the outputs must carry statistical reliability guarantees, and they must remain informative even when the observations are corrupted.

\subsection{Context and Motivation}\label{subsec:context}
Determining \emph{when} a system's behavior changes, and \emph{which} component is the root cause, is a recurring and often safety-critical task across engineering. In \emph{telecommunications}, abrupt shifts in traffic volume or flow statistics signal congestion, equipment failures, or intrusions, and detecting and localizing these shifts underpins network management, security monitoring, and the lifecycle of AI-based applications~\cite{aminikhanghahi2017survey,levyleduc2009detection,polese2023understanding,simeone2026conformal}.
In \emph{cybersecurity}, a large class of intrusion-detection tasks is naturally cast as changepoint detection, where a denial-of-service attack or a compromised host manifests as an abrupt change in packet-rate or connection statistics~\cite{tartakovsky2006novel}. In \emph{robotics}, fault detection and isolation must determine both the onset time of a sensor or actuator fault and the specific faulty component, so that a controller can react before the fault propagates through the platform~\cite{khalastchi2018fault}. In \emph{multi-agent and distributed systems}, from robot swarms to microservice architectures, a fault in a single agent or service can cascade through the collective, and the diagnostic problem is to detect the anomaly and trace it back to the agent or service that triggered it~\cite{soldani2022anomaly,zou2017nonparametric,cohen2015asymptotically}. In all of these domains, the \emph{changepoint} time is a proxy for the onset of a fault, and the \emph{root cause} is the component whose change occurred first.

\begin{figure}[t]
    \centering
    \includegraphics[width=\linewidth]{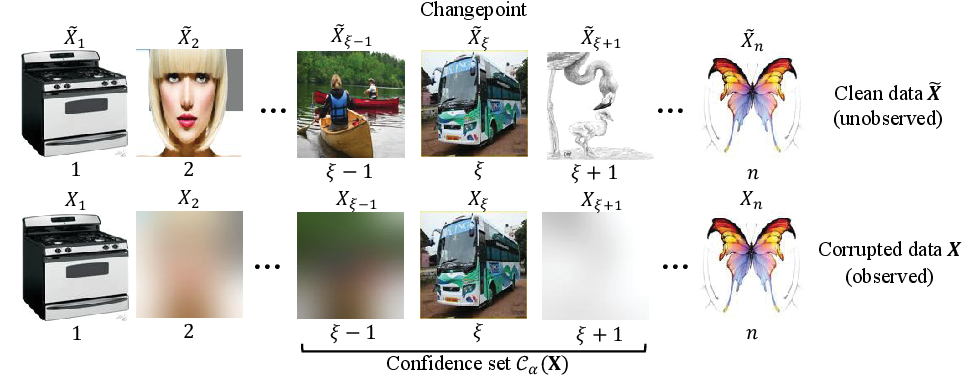}
    \caption{
    Illustration of the offline changepoint localization problem with contaminated data.
    The clean sequence $\widetilde{\mathbf X}$ changes at the true changepoint $\xi$,
    switching from one image style to another. 
    The clean sequence $\widetilde{\bX}$ is, however, not observable, as the changepoint detection has access only to the observed sequence $\mathbf X$, which is obtained by independently contaminating observations with unknown probability $\varepsilon$.
    The goal is to construct a confidence set $\mathcal C_\alpha(\mathbf X)$ that includes the true changepoint $\xi$ with probability no smaller than $1-\alpha$.}
    \label{fig:data_model1}
\end{figure}

A defining feature of these applications is that the output of a changepoint or root-cause procedure drives a potentially costly downstream action: rerouting traffic, quarantining a host, halting a robot, or isolating a microservice. A single point estimate, reported without any statement of confidence, is therefore of limited value, because an operator cannot tell whether to trust it. Recent work~\cite{kiyani2025decision} has shown that, for a \emph{risk-averse} decision maker, it is optimal to adopt decisions that maximize the given utility function for a worst-case outcome consistent with the observations. Specifically, the appropriate interface between uncertainty quantification and decision making is a \emph{set} of outcomes that provably contains the true one with probability at least $1-\alpha$, with probability $\alpha$ dictating the level of risk tolerance.

Returning a \emph{set} of plausible changepoints or root causes thus allows a decision maker to control its risk by selecting actions that are robust across all elements of the set. For example, an operator may inspect all the plausible causes of a failure in telecommunication networks on the basis of a set prediction. However, in order for this process to be efficient, it is of critical importance that the predicted set be as small as possible.

\begin{figure}[t]
    \centering
    \includegraphics[width=\linewidth]{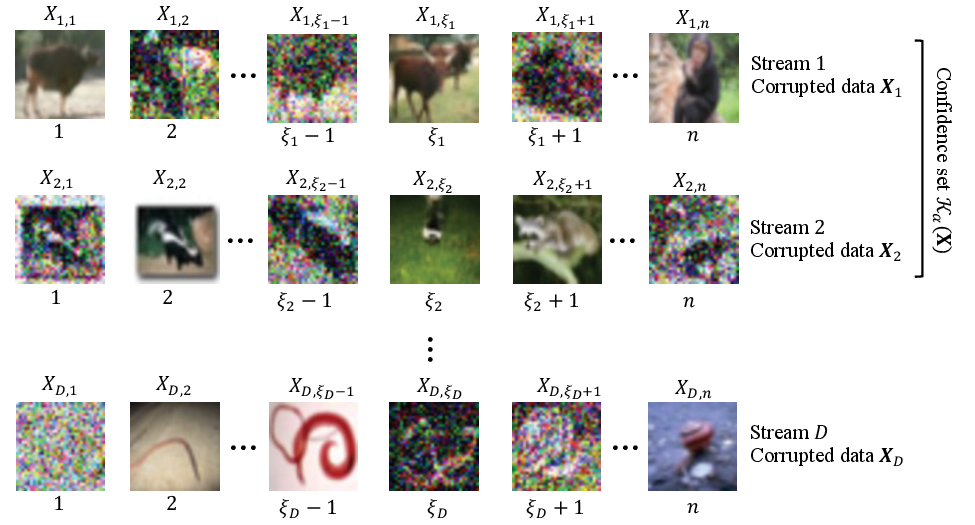}
    \caption{
    Illustration of the problem of offline root-cause localization.
    Each observed stream $\mathbf X_d$ with $d=1,\ldots,D$, undergoes a distributional shift at changepoint $\xi_d$, while some observations may be corrupted.
    The goal is to estimate the stream that has the earliest changepoint $d^\star$ by constructing a confidence set $\mathcal K_{\alpha}(\bX)$ that includes $d^\star$ with probability no smaller than $1-\alpha$.
    }
    \label{fig:data_model2}
\end{figure}

Real observations are routinely \emph{contaminated} by outliers, sensor faults, packet losses, or adversarial perturbations. Contamination can obscure the true change signal and true root cause of a change. Accordingly, even procedures whose coverage guarantee remains valid under contamination, such as \emph{conformal changepoint localization} (CONCH)~\cite{hore2026conformal} and \emph{conformal root cause analysis} (CROC)~\cite{hore2026distribution}, can be forced to return confidence sets so large as to be uninformative. The goal of this paper is to introduce a methodology that, building on CONCH and CROC, preserves validity under contamination while restoring informativeness of set predictors for both single-stream changepoint localization and multi-stream root cause analysis.

\subsection{Related Work}\label{subsec:related}
\subsubsection{Changepoint detection}
Classical changepoint methods include likelihood-ratio tests, CUSUM procedures, nonparametric tests, and kernel-based methods~\cite{page1955test,kim1989likelihood,pettitt1979non,song2024practical}, covering both offline and online settings~\cite{truong2020selective,aminikhanghahi2017survey,romano2023log}.
The theoretical guarantees of these standard methods typically rely on parametric assumptions, asymptotic approximations, or structural conditions, and most target detection or point localization rather than finite-sample confidence sets for the change location. The same limitations are shared by studies that address robustness to contamination, including adversarially robust offline detection under dynamic Huber contamination~\cite{li2021adversarially} and robust online mean-change detection under heavy-tailed noise with delay and false-alarm analysis~\cite{tang2026online}. In contrast to these studies, CONCH~\cite{hore2026conformal} produces a set of plausible changepoints that satisfies finite-sample coverage under a basic split-exchangeability condition, without requiring parametric assumptions. 

\subsubsection{Root cause analysis}
Root cause analysis has been developed largely within specific domains~\cite{sole2017survey,wang2023root}. 
For example, in microservice and distributed systems, causal-discovery and outlier-attribution methods localize the service responsible for a failure~\cite{soldani2022anomaly}, while in robotics fault detection and isolation identify faulty components or anomalous agents~\cite{khalastchi2018fault}. The recently proposed CROC~\cite{hore2026distribution} departs from this pattern by constructing distribution-free confidence sets for the stream whose changepoint occurs first.

\subsubsection{Conformal prediction} 
CONCH and CROC build on the general methodology of conformal prediction. 
Conformal prediction offers distribution-free, finite-sample predictive inference under exchangeability~\cite{vovk1999machine,shafer2008tutorial,romano2019conformalized}.
Extensions relax exchangeability or accommodate distribution shift, including weighted conformal prediction under covariate shift~\cite{tibshirani2019conformal,yoo2026calibrating}, conformal prediction beyond exchangeability~\cite{barber2023conformal}, and adaptive conformal inference for online shift~\cite{gibbs2021adaptive,zecchin2024localized}. Besides the offline settings considered by CONCH and CROC, conformal prediction has been adapted to online changepoint analysis through martingale tests~\cite{vovk2021testing}.

\subsubsection{Robustness and uncertainty estimation}
The contamination model adopted here follows Huber's framework~\cite{huber1964robust,hampel1986robust} and is related to data-poisoning settings~\cite{steinhardt2017certified}.
The weighting scheme relies on classifier uncertainty as a proxy for the probability that an observation is clean, drawing on evidential deep learning (EDL)~\cite{sensoy2018evidential}, Monte Carlo dropout~\cite{gal2016dropout}, deep ensembles~\cite{lakshminarayanan2017simple}, and calibration under shift~\cite{huang2025calibrating,simeone2022machine}. Finally, when the contamination level is unknown, we learn the uncertainty-to-weight mapping via meta-learning~\cite{finn2017model,chen2023learning},
optimizing a differentiable surrogate of the conformal set size in the spirit of conformal-aware training~\cite{stutz2021learning,park2023few}.

\subsection{Main Contributions}\label{subsec:contributions}
This paper develops offline changepoint detection and root cause analysis methods that remain valid and efficient even when observations are contaminated under Huber's model. Our contributions are as follows.
\begin{itemize}

\item \textbf{Weighted conformal changepoint localization.} We propose \emph{weighted CONCH} (W-CONCH), which builds on CONCH~\cite{hore2026conformal} by incorporating a weighting mechanism that downweights observations that are likely to be corrupted. W-CONCH hinges on a novel changepoint plausibility (CPP) score that is derived from a bound on the contaminated marginal likelihood, and is instantiated through an efficient classifier-based implementation.
\item \textbf{Uncertainty-based and meta-learned weights.} We introduce uncertainty-based weighting rules driven by classifier second-order uncertainty signals obtained from EDL or Bayesian learning. We also present a meta-learned variant, MW-CONCH, that directly optimizes a differentiable surrogate of the confidence set size, reducing dependence on hyperparameters.
\item \textbf{Extension to root cause analysis.} We extend the same weighting principle to multi-stream root-cause localization building on CROC \cite{hore2026distribution}. This yields weighted CROC (W-CROC) and its meta-learned variant MW-CROC, which preserve the distribution-free coverage guarantee of CROC while sharpening the root-cause confidence set.

\item \textbf{Empirical validation.} On image-based changepoint and root-cause benchmarks, along with a real-world changepoint benchmark, we show that uncertainty-based weighting substantially reduces confidence set size under contamination while retaining empirical coverage, with meta-learning achieving performance comparable to oracle baselines that use the true contamination indicators to define the weights.
\end{itemize}

\subsection{Organization}\label{subsec:organization}
The remainder of the paper is organized as follows. Section~\ref{sec:problem} formalizes the contaminated changepoint localization and root-cause localization problems. Section~\ref{sec:preliminaries} reviews CONCH and establishes its coverage guarantee under contaminated observations. Section~\ref{sec:method} presents W-CONCH, including the weighted changepoint-plausibility score, the uncertainty-based weighting rules, and the meta-learning procedure. Section~\ref{sec:wcroc} extends the weighting strategy to the multi-stream CROC framework. Section~\ref{sec:experiments} reports the experimental results, and the appendices contain proofs, implementation details, and additional experiments.

\section{Problem Setup}\label{sec:problem}
This paper studies changepoint localization and root cause analysis in a batch setting with corrupted observations.
As shown in Fig.~\ref{fig:data_model1}, in changepoint localization, the distribution of an ordered sequence of observations changes at a single unknown index, referred to as the changepoint.
The goal is to construct a confidence set for the changepoint that achieves finite-sample coverage while remaining as small as possible, despite possible data contamination.
As sketched in Fig.~\ref{fig:data_model2}, root cause analysis can be formulated as an extension of changepoint localization with multiple observed sequences in which one wishes to identify which sequence is the root cause of a change.
The goal here is to provide a set of possible root causes with coverage guarantees, even when the observations may be corrupted.
As we define in the following sections, the proposed techniques specialize the methodologies presented in~\cite{hore2026conformal} and~\cite{hore2026distribution} in order to account for data contamination, while retaining statistical validity.

\subsection{Changepoint Localization}
As illustrated in Fig.~\ref{fig:data_model1}, let $\mathcal X$ denote the observation space, and let $\widetilde{\bX} = (\widetilde{X}_1, \dots, \widetilde{X}_n) \in \mathcal X^n$ denote the {clean data sequence.} We assume that there exists a true changepoint $\xi \in \{1,\dots,n-1\}$ at which the distribution of the clean data $\widetilde{\bX}$ changes. As in~\cite{hore2026conformal}, we impose no parametric assumptions on the pre- and post-change distributions, requiring only the following exchangeability condition.

\begin{assumption}[Split exchangeability~\cite{hore2026conformal}]
\label{asm:split-exch}
    The clean sequence $\widetilde{X}$ is split-exchangeable at the true changepoint $\xi$: the pre- and post-change segments of the clean sequence $\widetilde{\bX}$ are exchangeable, i.e.,
    \begin{subequations}\label{eq:split-exch}
        \begin{align}
            ({\widetilde{X}}_1, \dots, {\widetilde{X}}_\xi)
            &\overset{\mathrm{d}}{=} ({\widetilde{X}}_{\pi_{0,\xi}(1)}, \dots, {\widetilde{X}}_{\pi_{0,\xi}(\xi)}), \\
            ({\widetilde{X}}_{\xi+1}, \dots, {\widetilde{X}}_n)
            &\overset{\mathrm{d} }{=} ({\widetilde{X}}_{\pi_{1,\xi}(\xi+1)}, \dots, {\widetilde{X}}_{\pi_{1,\xi}(n)}),
        \end{align}
    \end{subequations}
     where $\overset{d}{=}$ denotes equality in distribution, while $\pi_{0,\xi}(\cdot)$ and $\pi_{1,\xi}(\cdot)$ are arbitrary permutations of the sets $\{1,\dots,\xi\}$ and $\{\xi+1,\dots,n\}$, respectively.
\end{assumption}

Let $Y_1,\ldots, Y_n \overset{i.i.d.}{\sim} \mathrm{Bernoulli}(\varepsilon)$ be independent and identically distributed (i.i.d.) contamination indicators independent of the clean data $\widetilde{\bX}$ for some corruption probability $\varepsilon \in (0,1)$. Furthermore, let $Z_1, \dots, Z_n \overset{i.i.d.}{\sim} Q$ denote independent draws from a contamination distribution $Q$. 
{Following Huber's contamination model~\cite{huber1964robust}, each sample in the sequence of clean observations $\widetilde{\bX} = (\widetilde{X}_1,\ldots,\widetilde{X}_n)$ is independently corrupted with probability $\varepsilon$, yielding the observation
\begin{equation}\label{eq:contam_obs_model}
X_i = (1 - Y_i)\widetilde{X}_i + Y_i Z_i, \qquad i = 1, \dots, n.
\end{equation}
Accordingly, with probability $1-\varepsilon$, the observed sample $X_i$ equals the clean sample $\widetilde X_i$, while with probability $\varepsilon$, the observed sample $X_i$ is given by the noisy sample $Z_i\sim Q$.

Given the noisy observed sequence $\bX$ in~\eqref{eq:contam_obs_model}, we aim to construct a confidence set $\mathcal C_\alpha(\bX) \subseteq \{1,\dots,n-1\}$ for the unknown changepoint $\xi$ that satisfies the finite-sample coverage guarantee
\begin{align}\label{eq:target-objective}
    \Pr\bigl(\xi\in\mathcal C_\alpha(\bX)\bigr)\ge 1-\alpha
\end{align}
at a user-specified level $\alpha\in(0,1)$,
while reducing as much as possible the set size $|\mathcal C_\alpha(\bX)|$. This problem is studied in Section~\ref{sec:preliminaries} and Section~\ref{sec:method}.

\subsection{Root Cause Analysis}
\label{subsec:multi_stream_contamination_model}

As illustrated in Fig.~\ref{fig:data_model2}, consider now $D$ data streams, each of length $n$. 
For stream $d \in \{1,\ldots,D\}$, let $\widetilde{\bX}_d=(\widetilde X_{d,1},\ldots,\widetilde X_{d,n})$ denote the clean sequence, and let $\xi_d\in\{1,\ldots,n-1\}$ denote its changepoint. 
The vector of stream-wise changepoints is denoted by $\boldsymbol{\xi}=(\xi_1,\ldots,\xi_D)$.
The root-cause index is defined as the index of the stream whose changepoint occurs first, i.e.,
\begin{equation}\label{eq:root_cause_index}
d^\star=\arg\min_{d\in\{1,\ldots,D\}}\xi_d,
\end{equation}
where the minimizer is assumed to be unique. 
The rationale for this definition is that the first changepoint may be the root cause of the changes in the other sequences~\cite{hore2026distribution}.

Generalizing the single-stream setting, following~\cite{hore2026distribution}, we assume that, for each stream $d$ the clean sequence $\widetilde \bX_d$ satisfies split exchangeability at its true changepoint $\xi_d$. 

\begin{assumption}[Per-stream split exchangeability and independence]
\label{assump:multi_stream_split_exchangeability}
Any clean sequence $\widetilde{\mathbf X}_d$ is split-exchangeable at the changepoint $\xi_d$:
For any stream-wise split permutations
$\pi_{0,\xi_d}$ of the set $\{1,\ldots,\xi_d\}$ and
$\pi_{1,\xi_d}$ of the set $\{\xi_d+1,\ldots,n\}$, we have the equivalence in distribution
\begin{subequations}
\begin{align}
&(\widetilde X_{d,1},\ldots,\widetilde X_{d,\xi_d})
\overset{\mathrm{d}}{=}
(\widetilde X_{d,\pi_{0,\xi_d}(1)},\ldots,
\widetilde X_{d,\pi_{0,\xi_d}(\xi_d)}),\\
&\text{and }
(\widetilde X_{d,\xi_d+1},\ldots,\widetilde X_{d,n})
\overset{\mathrm{d}}{=}
(\widetilde X_{d,\pi_{1,\xi_d}(\xi_d+1)},\ldots,
\widetilde X_{d,\pi_{1,\xi_d}(n)}).
\end{align}
\end{subequations}
Furthermore, we assume that the clean streams $\widetilde \bX_{1},\widetilde \bX_{2},\ldots,\widetilde \bX_{D}$ are mutually independent.
\end{assumption}

Each stream is corrupted independently through a stream-wise Huber contamination mechanism. 
Accordingly, let $Y_{d,1},\ldots,Y_{d,n}\overset{i.i.d.}{\sim}\mathrm{Bernoulli}(\varepsilon_d)$ denote the contamination indicators for observation $i$ in stream $d$, and let $Z_{d,1},\ldots,Z_{d,n}\overset{i.i.d.}{\sim} Q_d$ denote an independent draw from a stream-specific contamination distribution $Q_d$. 
The $d$-th observed sequence is given by
\begin{equation}\label{eq:contam_obs_stream_model}
X_{d,i}=(1-Y_{d,i})\widetilde X_{d,i}+Y_{d,i}Z_{d,i}
\end{equation}
for observation index $i=1,\ldots,n$ and stream index $d=1,\ldots,D$.

Let $\bX=(\bX_1,\ldots,\bX_D)$ denote the collection of all observed streams.
The goal of root cause analysis is to construct a confidence set
$\mathcal K_{\alpha}(\bX)\subseteq\{1,\ldots,D\}$
for the unknown root-cause index $d^\star$ such that it includes the true index $d^\star$ in~\eqref{eq:root_cause_index} with probability no smaller than the user-defined level $1-\alpha$, i.e.,
\begin{equation}\label{eq:root_caused_confidence_set}
\Pr\bigl(
d^\star\in \mathcal K_{\alpha}(\bX)
\bigr)
\ge 1-\alpha.
\end{equation}
The root-cause localization problem is studied in Sec.~\ref{sec:wcroc}.

\section{Conformal Changepoint Localization with Contaminated Observations}
\label{sec:preliminaries}
In this section, we first review the conformal changepoint localization scheme introduced in~\cite{hore2026conformal}, which is referred to as 
CONCH.
Then, we show that the coverage guarantees of CONCH carry over to the contaminated observation model~\eqref{eq:contam_obs_model}.
Based on this result, in the next section we will develop the proposed W-CONCH.

\subsection{Conformal Changepoint Localization}
\label{subsec:conformal_changepoint_localization}
For each candidate changepoint $t\in\{1,\dots,n-1\}$, CONCH~\cite{hore2026conformal} evaluates a CPP score $S_t(\bX)$. 
This measures how likely time index $t$ is to be the true changepoint, with larger values indicating stronger evidence in favor of candidate $t$.
The CPP score is treated here as fixed and arbitrary, and we will discuss the optimization of the CPP score in the next section. 

Let $\Pi_t$ denote the set of permutations $\pi_{0,t}(\cdot)$ and $\pi_{1,t}(\cdot)$ that operate on the indices $\{1,\dots,t\}$ and $\{t+1,\dots,n\}$, respectively, as in Assumption~\ref{asm:split-exch}. 
Furthermore, denote as $\pi_t(\mathbf X)$ the permuted sequence $\big(X_{\pi_{0,t}(1)},\ldots,X_{\pi_{0,t}(t)},\allowbreak X_{\pi_{1,t}(t+1)},\ldots,X_{\pi_{1,t}(n)}\big)$ for some pair of permutations
$\pi_t = (\pi_{0,t}, \pi_{1,t})$.
The conformal $p$-value for candidate changepoint $t$ is defined as the fraction of permutations $\pi_t \in \Pi_t$ for which the CPP score $S_t(\pi_t(\mathbf X))$ does not exceed the actual CPP score $S_t(\mathbf X)$:
\begin{equation}\label{eq:pvalue-def}
  p_t(\bX) = \frac{1}{|\Pi_{t}|} \sum_{\pi_t\in\Pi_t}
    \Ind\bigl[S_t(\pi_t(\bX)) \leq S_t(\bX)\bigr].
\end{equation}
Intuitively, the statistic $p_t(\bX)$ in~\eqref{eq:pvalue-def} tends to be large if the CPP score $S_t(\mathbf X)$ is large, and thus there is evidence for $t$ being the true changepoint. 

In practice, when the number of data points $n$ is sufficiently large, enumerating all the split permutations in set ${\Pi}_t$ is computationally infeasible. 
In this case, a Monte Carlo (MC) variant of CONCH, referred to as CONCH-MC, replaces the full average over set $\Pi_t$ in~\eqref{eq:pvalue-def} with an average over randomly sampled split permutations.

The conformal $p$-value~\eqref{eq:pvalue-def}, and its MC version, are valid $p$-values for the null hypothesis $H_0:\xi=t$ that the true changepoint $\xi$ equals $t$.
Accordingly, given a miscoverage level $\alpha\in(0,1)$, the changepoint confidence set is obtained by inverting the test with null hypothesis $H_0 : \xi = t$~\cite{rice2007mathematical}, i.e., by retaining all candidates whose $p$-values exceed the level $\alpha$:
\begin{equation}\label{eq:cs}
  \mathcal{C}_\alpha(\bX)
  = \bigl\{ t \in \{1,\dots,n-1\} : p_t > \alpha \bigr\}.
\end{equation}
In other words, the confidence set retains all candidate changepoints that cannot be rejected by the split-permutation test with $p$-value~\eqref{eq:pvalue-def} at level $\alpha$.

\subsection{Coverage Guarantee of CONCH under Contamination}
Reference~\cite{hore2026conformal} shows that the CONCH set~\eqref{eq:cs} satisfies the marginal coverage guarantee~\eqref{eq:target-objective} as long as the sequence $\bX$ is split-exchangeable at the true changepoint $\xi$.
The following makes it possible to extend the coverage property also to the noisy observations~\eqref{eq:contam_obs_model} under Assumption~\ref{asm:split-exch}, which assumes split-exchangeability only for the clean sequence $\widetilde \bX$.
\begin{lemma}[Split exchangeability under contamination]\label{lem:coverage-exchangeability}
Under Assumption~\ref{asm:split-exch}, the corrupted observed sequence
$\bX$ generated according to~\eqref{eq:contam_obs_model} satisfies split exchangeability at the true changepoint $\xi$.
\end{lemma}
\begin{proof}
See Appendix~\ref{app:proof-exchangeability-contamination}.    
\end{proof}
This result has the following immediate consequence.

\begin{lemma}[Coverage guarantee under contamination]\label{lem:coverage-contamination}
    Under Assumption~\ref{asm:split-exch}, the CONCH set $\mathcal C_\alpha(\bX)$ defined in~\eqref{eq:cs}, when applied to the observed corrupted sequence $\bX$ in~\eqref{eq:contam_obs_model}, satisfies the required coverage condition~\eqref{eq:target-objective}.
\end{lemma}

\begin{proof}
By Lemma~\ref{lem:coverage-exchangeability}, the claim follows directly from \cite[Thm.~3.1]{hore2026conformal}.    
\end{proof}

Lemma~\ref{lem:coverage-contamination} ensures that the CONCH set~\eqref{eq:cs} meets the coverage requirement in~\eqref{eq:target-objective} regardless of the contamination level $\varepsilon$. 
However, unless the CPP score $S_t(X)$ is properly designed, contaminated observations can obscure the changepoint signal and inflate the confidence set (see Fig.~\ref{fig:data_model1}).
We address this efficiency issue by introducing novel CPP scores in the next section.

\section{Weighted Conformal Changepoint Localization}
\label{sec:method}
Lemma~\ref{lem:coverage-contamination} ensures that the CONCH set~\eqref{eq:cs} remains valid under the contamination model~\eqref{eq:contam_obs_model}, but contaminated observations generally increase the size $|\mathcal{C}_\alpha(\bX)|$ of the confidence set, making it more difficult in identifying the changepoint.
To address this issue, in this section we introduce CPP scores that aim at downweighting observations that are likely to be corrupted. 
To this end, we first consider an oracle-based setting in which the model~\eqref{eq:split-exch}--\eqref{eq:contam_obs_model} is known, and then introduce a practical implementation based on uncertainty signals.
Finally, we present a meta-learning-based strategy that leverages offline data from multiple tasks to further improve the CPP score.

\subsection{Weighted CPP Score}
\label{subsec:weighted-score}
In this section, we consider a simplified oracle-based setting in order to facilitate the derivation of an optimal CPP score under contamination.
To this end, following~\cite{hore2026conformal}, we specialize the data model in Sec.~\ref{sec:problem} to a standard i.i.d.\ changepoint model with known probability distribution~{\cite{kim1989likelihood,pettitt1979non, dandapanthula2025offline}}.
\begin{assumption}[i.i.d. oracle changepoint model]
\label{assump:iid_changepoint_model}
The clean pre- and post-change distributions admit known densities $f_0$ and $f_1$, so that the pre-change samples
$\widetilde X_1,\ldots,\widetilde X_\xi$
are drawn i.i.d. from a distribution with density $f_0(x)$ and post-change samples $\widetilde X_{\xi+1},\ldots,\widetilde X_n$ are drawn i.i.d.\ from a distribution with density $f_1(x)$:
\begin{equation}\label{eq:iid_chnagepoint_model}
    \widetilde X_1,\ldots,\widetilde X_\xi \overset{i.i.d}{\sim} f_{0}(x),\quad 
    \widetilde X_{\xi+1},\ldots,\widetilde X_n \overset{i.i.d}{\sim} f_{1}(x).
\end{equation}
Furthermore, each observation $X_i$ is independently drawn from the contaminated marginal density 
\begin{equation}\label{eq:contam_marginal}
    g_j(x) = (1-\varepsilon) f_j(x) + \varepsilon q(x), 
\end{equation}
with $j=0$ for pre-change samples $i=1,\ldots,\xi$ and $j=1$ for post-change samples $i=\xi+1,\ldots,n$, where probability $\varepsilon$ and density $q(x)$ are known.
\end{assumption}

\subsubsection{Optimal Oracle CPP Score}
\label{subsubsec:optimal_oracle_score}

Under Assumption~\ref{assump:iid_changepoint_model}, the marginal densities
$g_0(x)$ and $g_1(x)$ in~\eqref{eq:contam_marginal} are known, and so is the true changepoint $\xi$.
Under this idealized assumption~\ref{asm:split-exch}, the following likelihood-ratio statistic can be proved to yield the optimal CPP score that minimizes the expected size of the confidence set $\mathcal{C}_{\alpha}(\bX)$.

\begin{lemma}[Optimal CPP score under contamination]
\label{lem:opt_score_contam}
Under Assumption~\ref{assump:iid_changepoint_model}, the CPP score
\begin{equation}\label{eq:opt_score_contam}
S^{\mathrm{opt}}_t(\bX)
=
\frac{\prod_{i\le t} g_0(X_i)\prod_{i>t} g_1(X_i)}
     {\prod_{i\le \xi} g_0(X_i)\prod_{i>\xi} g_1(X_i)}
\end{equation}
attains the minimum average set size $|\mathcal{C}_{\alpha}(\bX)|$ among all CONCH sets defined in~\eqref{eq:cs}.
\end{lemma}
\begin{proof}
Follows directly from \cite[Thm.~4.3]{hore2026conformal}.    
\end{proof}

\subsubsection{Weighted CPP Score}
\label{subsubsec:weighted-score}
The score in~\eqref{eq:opt_score_contam} is theoretically optimal but practically inaccessible, since the contaminated marginals {$g_0(x)$ and $g_1(x)$} in~\eqref{eq:contam_marginal} depend on a number of unknowns, namely the clean densities {$f_0(x)$ and $f_1(x)$}, the contamination level $\varepsilon$, and the contaminated samples density $q(x)$, while the denominator further requires the unknown true changepoint $\xi$. To obtain an implementable score, we make the following substitutions: {1)} each unknown clean density $f_j(x)$ is replaced by a learned estimate $\hat f_{j}(x)$; {2)} the true changepoint $\xi$ in the denominator is replaced by an estimate $\hat\xi(\bX)$; and {3)} each contaminated likelihood $g_j(x)$ is replaced by a tractable bound.

The approximations 1) and 2) follow reference~\cite{hore2026conformal}, and they will be further discussed below for completeness. In contrast, the approximation {3)} is specific to our proposed CPP score, and is derived next via a bound on the distribution~\eqref{eq:contam_marginal}.

\begin{proposition}\label{prop:elbo_bound}
Assume that the observation space $\mathcal X$ is bounded, and let the contamination density $q$ is bounded away from zero and infinity on its support, i.e., $|\log q(x)|\le C_q$ for some constant $C_q<\infty$ for all $x$.
Then, the marginal density $g_j(x)$ in~\eqref{eq:contam_marginal} can be lower bounded as
\begin{equation}\label{eq:elbo_bound}
    g_j(x) \;\ge\; C \cdot f_j(x)^{w_j(x)},
\end{equation}
for $j=0,1$ where $w_j:\mathcal X \to [0,1]$ is a weight function, and $C$ is a constant depending only on $C_q$ and $\varepsilon$.
\end{proposition}
\begin{proof}
As detailed in Appendix~\ref{app:elbo_bound}, the inequality~\eqref{eq:elbo_bound} follows by applying the evidence lower bound (ELBO)~\cite{simeone2022machine} to the log-marginal $\log g_j(x)$ and then bounding the resulting remainder term using the stated boundedness assumption.    
\end{proof}

Substituting the bound~\eqref{eq:elbo_bound} for each unknown density $g_j(X_i)$ in the optimal score~\eqref{eq:opt_score_contam}, as well as the unknown clean densities $f_j(x)$ by estimates $\hat f_j(x)$ and the true changepoint $\xi$ by an estimate $\hat\xi(\bX)$ yields the proposed weighted CPP score defined next.

\begin{definition}[Weighted CPP Score]\label{def:weighted_cpp}
Given density estimates $\hat f_{j}(x)$ for $j=0, 1$, changepoint locator estimate $\hat{\xi}(\mathbf X)$, and arbitrary weight functions $w_0(\mathbf X)$ and $w_1(\mathbf X)$, the weighted CPP score is defined as
\begin{equation}
\label{eq:weighted_cpp_score}
\begin{aligned}
S_t^w(\bX)
&=
\log\!\left(
\frac{
\prod_{i \le t} \hat f_{0}(X_i)^{w_0(X_i)}
\prod_{i > t} \hat f_{1}(X_i)^{w_1(X_i)}
}{
\prod_{i \le \hat{\xi}(\bX)} \hat f_{0}(X_i)^{w_0(X_i)}
\prod_{i > \hat{\xi}(\bX)} \hat f_{1}(X_i)^{w_1(X_i)}
}
\right).
\end{aligned}
\end{equation}
\end{definition}

In the weighted CPP score~\eqref{eq:weighted_cpp_score}, each likelihood contribution $\hat f_{j}(X_i)$ is modulated by its weight $w_j(X_i)$, so that observations with small {weight $w_j(X_i)$} have reduced influence on the score. As shown in the proof of Proposition~\ref{prop:elbo_bound}, the weight $w_j(x)$ in~\eqref{eq:elbo_bound} ideally corresponds to the posterior probability that the corrupted sample $X_i=(1-Y_i)\widetilde X_i + Y_i Z_i$ in~\eqref{eq:contam_obs_model} is clean, i.e., that we have $Y_i=0$. Thus, intuitively, the introduction of the weight $w_j(x)$ allows the score~\eqref{eq:weighted_cpp_score} to selectively gives more relevance to samples $X_i$ that are more likely to be clean. Sec.~\ref{sec:weighting-construction} will discuss a practical implementation of the score~\eqref{eq:weighted_cpp_score} in which the weights $\{w_0(\cdot), w_1(\cdot)\}$ are extracted from uncertainty signals produced by a classifier. This construction builds on a {classifier-based} version of the score~\eqref{eq:weighted_cpp_score}, which is discussed next.

\subsubsection{Classifier-Based CPP Score}
\label{subsubsec:classifier-based_score}
The weighted CPP score~\eqref{eq:weighted_cpp_score} can be implemented using different estimates of the densities $f_0(x)$, and ${f}_1(x)$, and of the changepoint ${\xi}$. 
Following~\cite{hore2026conformal}, instead of relying on separate estimates $f_0(x)$ and $f_1(x)$ of the clean data distributions, here we assume access to a binary classifier $\{\hat p(j\mid x)\}_{j=0,1}$, trained to distinguish between clean pre-change samples $\widetilde \bX \sim f_0(x)$, labeled as $j=0$, from clean post-change samples $\widetilde \bX \sim f_1(x)$, labeled as $j=1$~\cite{joseph2018adversarial}.
Note that the classifier is trained offline using clean labeled data that is independent of the current batch $X$.
Given such a classifier, write the classifier log-posterior ratio
\begin{equation}
\label{eq:classifier_log_posterior_contrast}
\Delta_i
=
\log\frac{\hat p(j=0\mid X_i)}
{\hat p(j=1\mid X_i)}.
\end{equation}
When the labels $j=0$ and $j=1$ are a priori equally likely, the ratio~\eqref{eq:classifier_log_posterior_contrast} approximates the log-likelihood ratio as $\Delta_i\approx \log (f_0(X_i)/f_1(X_i))$~\cite{simeone2022machine,joseph2018adversarial}.

Using this approximation, the maximum likelihood estimation (MLE) of the changepoint $\xi$ can be expressed as~\cite{hore2026conformal},
$\hat{\xi}(\bX) = \arg\max_s \sum_{i\le s} \Delta_i$.
Under contamination, we analogously define the corresponding weighted log-likelihood and weighted MLE.
Furthermore, using the same approximation together with the weighted estimate~\eqref{eq:weighted_mle} in~\eqref{eq:weighted_cpp_score} and canceling the terms that do not affect the maximization over the reference changepoint yield the following classifier-based weighted CPP score.

\begin{definition}[Classifier-Based Weighted CPP Score]\label{def:clf_weighted_cpp}
Given a classifier $\{\hat p(j\mid x)\}_{j=0,1}$ with log-posterior ratio
$\Delta_i$ as in~\eqref{eq:classifier_log_posterior_contrast}, and given weight
functions $w_j:\mathcal X\to[0,1]$, let
$w^{(t)}(X_i)=w_0(X_i)$ for $i\le t$ and
$w^{(t)}(X_i)=w_1(X_i)$ for $i>t$, and define
the corresponding weighted log-likelihood
\begin{equation}
N_t(s) = \sum_{i\le s}w^{(t)}(X_i)\log\hat p(j=0\mid X_i) + \sum_{i>s}w^{(t)}(X_i)\log\hat p(j=1\mid X_i).    
\end{equation}\label{eq:weighted_loglik}
The corresponding weighted MLE is
\begin{equation}\label{eq:weighted_mle}
\hat\xi_t(\bX) = \arg\max_{s\in\{1,\ldots n-1\}}N_t(s).
\end{equation}
The classifier-based weighted CPP score is defined as
\begin{equation}\label{eq:weighted_cusum}
S_t^w(\bX)
=
\sum_{i\le t}w_0(X_i)\Delta_i
-
\sum_{i\le \hat\xi_t(\bX)}w^{(t)}(X_i)\Delta_i.
\end{equation}
\end{definition}

The derivation of the weighted CPP score~\eqref{eq:weighted_cusum} is provided in Appendix~\ref{app:classifier-cusum-derivation}.
The first term in~\eqref{eq:weighted_cusum} represents the cumulative weighted log-posterior ratio up to the candidate changepoint $t$.
The second term represents the cumulative weighted log-posterior ratio evaluated at the weighted MLE $\hat\xi_t(\bX)$, with the latter corresponding to the most plausible competing changepoint.
Thus, the CPP $S_t^w(\bX)$ measures the relative plausibility of candidate $t$ against the best competing changepoint.

\subsection{Uncertainty-Based Weights}
\label{sec:weighting-construction}
As discussed in Sec.~\ref{subsec:weighted-score}, the weight $w_j(x)$ used
in the CPP score~\eqref{eq:weighted_cpp_score} or~\eqref{eq:weighted_cusum}
should ideally capture the true probability that a sample $X \sim g_j(x)$ is clean. 
To approximate this probability, which depends on the unknown distribution~\eqref{eq:elbo_bound}, we propose to use classification uncertainty as a proxy. 
The rationale for this choice is that a classifier model $\hat{p}(j\mid x)$ trained on clean data should ideally be confident on clean in-distribution samples and uncertain on contaminated ones~\cite{huang2025calibrating, malinin2018predictive}.
With this approach, low-uncertainty observations $X_i$ receive larger weights, while high-uncertainty observations receive smaller weights.

To elaborate, as in Sec.~\ref{subsubsec:classifier-based_score}, consider a classifier $\hat p(j\mid x)$ pre-trained to distinguish between clean samples from the pre-change density $f_0(x)$, labeled with $j=0$ and the post-change density $f_1(x)$, labeled as $j=1$.
In addition to the class posterior scores $\{\hat p(j\mid x)\}_{j=0,1}$, which are used to construct the classifier-based CPP score~\eqref{eq:weighted_cusum}, we assume that the classifier provides an uncertainty signal
\begin{equation}\label{eq:uncertainty_score}
M_i=h(X_i),    
\end{equation}
for some function $h:\mathcal X\to\mathbb R_{\ge 0}$, where larger values of $M_i$ indicate higher classification uncertainty.
For example, the function $h(\cdot)$ can be obtained from an evidential classifier, which uses the parameters of a predictive Beta distribution to quantify second-order uncertainty~\cite{sensoy2018evidential}, i.e., uncertainty about the predictive distribution, or using Bayesian methods, such as MC dropout~\cite{gal2016dropout}, which estimate such uncertainty via ensembling~\cite{ lakshminarayanan2017simple,simeone2022machine}.

A natural choice is to assign weight $1$ to observations below an uncertainty threshold $\kappa_j$ and weight $0$ to those above, yielding the hard thresholding rule
\begin{equation}\label{eq:binary_weight_unc}
w_j^{\mathrm{hard}}(X_i)
=
\begin{cases}
1, & M_i \le \kappa_j,\\
0, & \text{otherwise}
\end{cases}
\end{equation}
for $j=0, 1$.
A smoother alternative replaces the hard threshold with a sigmoid function, yielding the soft weighting rule
\begin{equation}\label{eq:soft_unc_weight}
w_j^{\mathrm{soft}}(X_i)
=
\sigma\left(\frac{\kappa_j-M_i}{\lambda}\right),
\quad \lambda>0,
\end{equation}
where $\sigma(\cdot)$ is the sigmoid function and hyperparameter $\lambda$ controls the sharpness of the transition. As $\lambda\to0$, the soft rule recovers the hard thresholding rule in~\eqref{eq:binary_weight_unc}.

The thresholds $\kappa_j$ for $j=0,1$ in~\eqref{eq:binary_weight_unc} should ideally capture the transition between typical uncertainty signals for clean samples and for corrupted samples on either side of the changepoint. 
Accordingly, we propose to use side-specific empirical quantiles of the uncertainty signals.
Formally, for a candidate changepoint $t$ and a quantile hyperparameter $\beta\in[0,1]$, we set
\begin{subequations}\label{eq:side_quantiles}
    \begin{align}
    \kappa_0(t,\beta)
    &=
    \mathrm{Q}_{1-\beta}\bigl(\{M_1,\dots,M_t\}\bigr), \\
    \text{and }\kappa_1(t,\beta)
    &=
    \mathrm{Q}_{1-\beta}\bigl(\{M_{t+1},\dots,M_n\}\bigr),
    \end{align}
\end{subequations}
where $Q_{1-\beta}(\mathcal{S})$ returns the $\lfloor (1-\beta)|\mathcal S| \rfloor$ smallest element in set $\mathcal{S}$.
The hyperparameter $\beta$ controls should ideally match the expected fraction $\varepsilon$ of contaminated observations.
Based on this, we recommend setting hyperparameter $\beta$ to one’s best conservative estimate of the contamination probability $\varepsilon$. 

\subsection{Coverage Properties of W-CONCH}
\label{sec:coverage_comparison_wconch}
Overall, specializing the CONCH set~\eqref{eq:cs}, the W-CONCH set is defined as follows.
\begin{definition}[W-CONCH Confidence Set]\label{def:wconch}
Given hyperparameters $\lambda>0$ and $\beta\in[0,1]$, which control the sharpness of the soft weighting rule~\eqref{eq:soft_unc_weight} and the quantile level for the side-specific uncertainty thresholds~\eqref{eq:side_quantiles}, respectively, W-CONCH scheme produces the confidence set
\begin{equation}\label{eq:wconch_confidence_set}
\begin{aligned}
\mathcal C_\alpha^w(\bX)
&=
\bigl\{t \in \{1,\ldots,n-1\}: p_t^w(\bX) > \alpha\bigr\}, \\
{\textrm{with}}~p_t^w(\bX)
&=
\frac{1}{|\Pi_t|}
\sum_{\pi_t \in \Pi_t}
\Ind\bigl[
S_t^w(\pi_t(\bX)) \leq S_t^w(\bX)
\bigr],
\end{aligned}
\end{equation}
where the weighted score $S_t^w(\bX)$ in~\eqref{eq:weighted_cusum} is implemented using the weights defined in~\eqref{eq:soft_unc_weight}--\eqref{eq:side_quantiles}.
\end{definition}
The following proposition summarizes the properties of the W-CONCH set~\eqref{eq:wconch_confidence_set}.
\begin{proposition}[Marginal coverage of W-CONCH]
\label{prop:weighted-coverage}
Under Assumption~\ref{asm:split-exch}, for any choice of hyperparameter $\lambda$ and $\beta$, the W-CONCH set
$\mathcal C_{\alpha}^w(\mathbf X)$ in~\eqref{eq:wconch_confidence_set} satisfies the marginal coverage condition
\begin{equation}\label{eq:weight_coverage}
\Pr\bigl(\xi\in \mathcal C_{\alpha}^w(\mathbf X)\bigr)\ge 1-\alpha.
\end{equation}
\end{proposition}

\begin{proof}
As detailed in Appendix~\ref{app:proof-weighted-coverage}, the result follows from the split-permutation invariance of the weighted score~\eqref{eq:weighted_cusum} and the finite-sample rank argument of~\cite[Thm.~3.1]{hore2026conformal}.    
\end{proof}

\subsection{Meta-Learned Uncertainty Weights}
\label{subsec:meta_uncertainty_weighting}
The W-CONCH set~\eqref{eq:wconch_confidence_set} in Sec.~\ref{sec:weighting-construction} requires choosing the quantile hyperparameter $\beta$, as well as the threshold hyperparameter $\lambda$. 
As discussed in the previous subsection, when the contamination level $\varepsilon$ is known, the hyperparameter $\beta$ can be naturally set to $\varepsilon$, but a reasonable estimate of the probability $\varepsilon$ may not be available. 
Moreover, a pretrained uncertainty estimator $h(x)$ may not necessarily produce scores $M_i$ in~\eqref{eq:uncertainty_score} that are well aligned with the objective of minimizing the average set size $|\mathcal C_\alpha^w(\bX)|$. 
We address both issues by introducing a meta-learning strategy that optimizes a task-adapted uncertainty estimator from a collection of contaminated changepoint tasks, so that the induced weights directly minimize the size objective $|\mathcal C_\alpha^w(\bX)|$.
We refer to this meta-learned variant of W-CONCH as MW-CONCH.
Fig.~\ref{fig:meta_uncertainty_pipeline} illustrates the overall pipeline of MW-CONCH.

\begin{figure}[t]
    \centering
    \includegraphics[width=\linewidth]{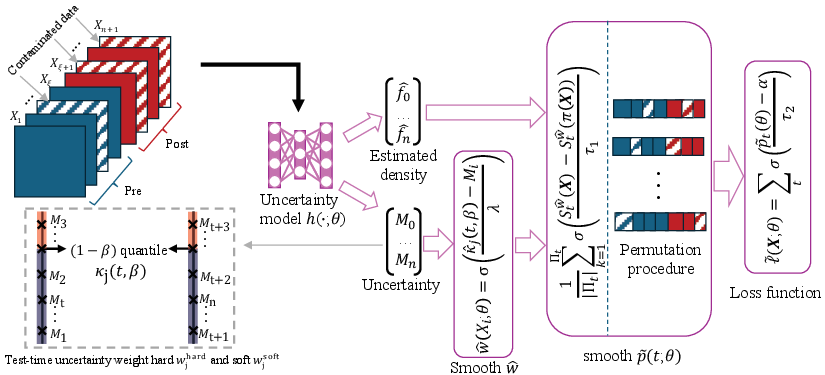}
    \caption{
    Operation of MW-CONCH: The uncertainty model $h(\cdot;\theta)$ produces per-observation uncertainty scores $M_i$, which are converted to weights via the smooth weighting rule. The weighted CPP score~{\eqref{eq:weighted_cpp_score}} and the permutation procedure of Sec.~{\ref{sec:preliminaries}} together yield a smoothed $p$-value, from which a differentiable surrogate of the confidence set size is computed as the training loss. Estimated pre- and post-change densities $\hat f_{j}$ enter the CPP score independently of $\theta$.
    }
    \label{fig:meta_uncertainty_pipeline}
\end{figure}
To start, MW-CONCH selects a parameterized uncertainty estimator $h(\cdot;\theta):\mathcal X\to\mathbb R_{\ge 0}$ with trainable parameters $\theta$, replacing the pretrained model $h$ in Sec.~{\ref{sec:weighting-construction}}. 
For example, the model $h(\cdot;\theta)$ may be 
a neural network trained via EDL~\cite{sensoy2018evidential}. 
EDL places a Dirichlet prior $\text{Dir}(\boldsymbol{e})$ over class probabilities and an uncertainty score $h(X_i;\theta) = J / \sum_{j=1}^{J} e_j$, where $J$ is the number of classes and $\boldsymbol{e} = g(X_i;\theta)$ are the predicted concentration parameters.
The goal is to optimize the parameter $\theta$ so as to minimize the average size of the confidence set $|\mathcal C_\alpha^w(\bX)|$.

To this end, we fix the distributions $f_0(\cdot)$ and $f_1(\cdot)$ of the clean data, and we assume to have access to data $\bX^k = (X_1^k,\ldots,X_n^k)$ generated i.i.d. from some distribution satisfying Assumption~\ref{asm:split-exch} and following the contamination model~\eqref{eq:contam_obs_model} for some randomly drawn contamination level $\varepsilon_{\mathrm{tr}}\sim \mu$ from some distribution $\mu$.
The training distribution $\mu$ is itself a design choice: it may concentrate on a single reference contamination level $\varepsilon_{\mathrm{tr}}$, which is referred to as fixed-$\varepsilon_{\mathrm{tr}}$ meta-learning, or it may spread mass over multiple levels, which is referred to as mixed-$\varepsilon_{\mathrm{tr}}$ meta-learning. 
The meta-training loss is the average confidence set size of the W-CONCH set~\eqref{eq:wconch_confidence_set} over the samples $\{\bX^{k}\}_{k=1}^K$:
\begin{equation}\label{eq:meta-loss}
\mathcal L(\theta)
=
\frac{1}{K}\sum_{k=1}^K
\left|
\mathcal C_{\alpha;\theta}^{w}(\bX^{k})
\right|,
\end{equation}
where we have made explicit the dependence of the W-CONCH set $\mathcal C_{\alpha;\theta}^{w}(\bX)$ on the parametric uncertainty score $h(\cdot;\theta)$.

The meta-learning loss in~\eqref{eq:meta-loss} is not directly differentiable and a smooth approximation is derived in Appendix~\ref{app:algorithm} by following an approach similar to~\cite{park2023few,stutz2021learning}. 
MW-CONCH optimizes this differentiable objective via gradient descent.

\subsection{Incorporating Prior Information on the Changepoint Location}
\label{sec:wconch_prior}
When prior information is available on which candidate locations are more likely to be the true changepoint $\xi$, it can be used to further concentrate the confidence set. We encode the prior as weights $v_1,\dots,v_{n-1}\ge 0$ over the candidate locations. The weights are normalized as $\sum_{t=1}^{n-1}v_t=n-1$, so that values $v_t>1$ correspond to locations deemed more likely than under a uniform prior. Unlike the weights $w_0(X_i),w_1(X_i)$ in~\eqref{eq:weighted_cpp_score}, which act on the observations, the prior weights $\{v_t\}_{t=1}^{n-1}$ act on the candidate locations, entering through the significance level used to test each candidate.

Specifically, following the level-allocation strategy of weighted hypothesis testing~\cite{genovese2006false}, each candidate $t$ is tested at the level
\begin{equation}\label{eq:prior_level}
    \alpha_t=\min\left(\frac{\alpha}{v_t},\,\alpha_{\max}\right),
\end{equation}
where $\alpha_{\max}\in[\alpha,1)$ bounds the level applied to any single location, yielding the prior-informed confidence set
\begin{equation}\label{eq:prior_confidence_set}
    \mathcal C_{\alpha,v}^{w}(\bX)=\bigl\{t\in\{1,\dots,n-1\}: p_t^w(\bX)>\alpha_t\bigr\},
\end{equation}
with the $p$-value $p_t^w(\bX)$ computed as in~\eqref{eq:wconch_confidence_set}.
By~\eqref{eq:prior_level}, a priori likely locations are tested at smaller levels, and are thus harder to exclude, while unlikely locations are more easily discarded. Under a uniform prior, i.e., $v_t=1$ for all $t$, all levels reduce to $\alpha_t=\alpha$, recovering the W-CONCH set~\eqref{eq:wconch_confidence_set}.

\begin{proposition}[Coverage of prior-informed W-CONCH]
\label{prop:prior-coverage}
Under Assumption~\ref{asm:split-exch}, the set $\mathcal C_{\alpha,v}^{w}(\bX)$ in~\eqref{eq:prior_confidence_set} satisfies 
\begin{equation}\label{eq:prior_avg_coverage}
 \Pr\bigl(\xi\in\mathcal C_{\alpha,v}^{w}(\bX)\bigr)\ge 1-\alpha_{\max}   
\end{equation}
for any changepoint $\xi$. 
\end{proposition}
\begin{proof}
See Appendix~\ref{app:proof-prior-coverage}.    
\end{proof}

The parameter $\alpha_{\max}$ controls the influence of the prior. The worst-case bound $1-\alpha_{\max}$ in~\eqref{eq:prior_avg_coverage} protects against misspecification of the prior, with the choice $\alpha_{\max}=\alpha$ recovering the guarantee~\eqref{eq:weight_coverage} of Proposition~\ref{prop:weighted-coverage}.

\section{Weighted Conformal Root-Cause Localization with Contaminated Observations}
\label{sec:wcroc}

In this section, we introduce an extension of W-CONCH to the multi-stream root-cause localization setting under contaminated observations described in Sec.~\ref{subsec:multi_stream_contamination_model}.
The corresponding framework, referred to as W-CROC, constructs a confidence set $\mathcal{K}_{\alpha}(\bX)$ for the root-cause index by applying the CROC scheme presented in~\cite{hore2026distribution} with a weighted CPP score analogous to W-CONCH.

\subsection{Conformal Root Cause Analysis}
\label{subsec:croc}

Unlike the single-stream setting studied in the previous section, the unknown changepoint is now a vector
$\boldsymbol{\xi}=(\xi_1,\ldots,\xi_D)$, and the root-cause index $d^{\star}$ in~\eqref{eq:root_caused_confidence_set} is determined by the stream with the earliest changepoint.
Let $\bt=(t_1,\ldots,t_D)$ be a candidate changepoint configuration in a feasible set $\mathcal R$.
The set $\mathcal R$ encodes prior structural constraints on the changepoints that may be a priori known to the detector. 
For instance, set $\mathcal R$ may consist of configurations in which a single stream changes first, i.e.,
$ \mathcal R = \{(t_1,\ldots,t_D): \exists d\in\{1,\ldots,D\}\ \text{such that}\ t_d<t_m,\ \forall m\ne d\}$.
For each candidate configuration $\bt\in\mathcal R$, CROC evaluates a CPP score $S_{\bt}(\bX)$ that measures the plausibility of $\bt$ being the true changepoint configuration.

To calibrate this score, CROC applies split permutations within each stream $d$ using the candidate changepoint $t_d$.
Let $\Pi_{\bt}$ denote the corresponding split-permutation group, where each stream $d$ is permuted only within its candidate pre- and post-change segments determined by candidate $t_d$.
The configuration-level conformal $p$-value is defined as
\begin{equation}
\label{eq:croc_config_pvalue}
p_{\bt}
=
\frac{1}{|\Pi_{\bt}|}
\sum_{\pi\in\Pi_{\bt}}
\mathbf 1
\left\{
S_{\bt}(\pi(\bX))
\le
S_{\bt}(\bX)
\right\}.
\end{equation}
In practice, the average in~\eqref{eq:croc_config_pvalue} can be approximated using MC split permutations as explained in Sec.~\ref{subsec:conformal_changepoint_localization}.

To obtain a confidence set for the root-cause index, CROC aggregates the configuration-level $p$-values over all configurations that are consistent with each candidate root stream.
To elaborate, for each $d\in\{1,\ldots,D\}$, define the set
$\mathcal I_d=\{\bt\in\mathcal R : t_d < t_m \text{ for all } m\neq d\},$
which contains all feasible configurations under which stream $d$ is the root-cause stream.
The root-cause $p$-value for stream $d$ is then computed by taking the maximum over all feasible configurations in which stream $d$ changes first across all other streams, i.e.,
\begin{equation}
\label{eq:croc_root_pvalue}
p_{(d)}
=
\max_{\bt\in\mathcal I_d}
p_{\bt}.
\end{equation}
Then, stream $d$ is included in the confidence set $\mathcal{K}_{\alpha}(\bX)$ if at least one feasible configuration in which stream $d$ changes first is not rejected.
Formally, given a miscoverage level $\alpha\in(0,1)$, the CROC confidence set for the root-cause index is
\begin{equation}
\label{eq:croc_confidence_set}
\mathcal K_{\alpha}(\bX)
=
\{d\in\{1,\ldots,D\}: p_{(d)}>\alpha\}.
\end{equation}

\subsection{Weighted Conformal Root Cause Analysis}
\label{subsec:weighted_croc_score}
W-CROC specializes CROC by replacing the CPP score with the weighted classifier-based score developed in Sec.~\ref{sec:method}, in order to account for the contaminated observations in~\eqref{eq:contam_obs_stream_model}.
Assume access to a pre-trained classifier for pre-change and post-change samples for each stream $d$, so that a log-posterior ratio $\Delta_{d,i}$, as in~\eqref{eq:classifier_log_posterior_contrast}, can be computed for each stream $d$.
\begin{definition}[Weighted CROC Score]\label{def:weighted_croc}
Given existing weights $w_{d,j}:\mathcal X\to[0,1]$ for $d=1,\ldots,D$, the configuration level weighted CPP score for a candidate changepoint configuration $\bt=(t_1,\ldots,t_D)$$\in \mathcal R$ is defined as
\begin{equation}
\label{eq:weighted_croc_score}
S_{\bt}^{w}(\bX)
=
\sum_{d=1}^{D}
S_{t_d}^{w}(\bX_d),
\end{equation}
where $S_{t_d}^{w}(\bX_d)$ is the single-stream weighted CPP score~\eqref{eq:weighted_cusum} applied to stream $d$, which is given by
\begin{equation}
\label{eq:weighted_croc_stream_score}
S_{t_d}^{w}(\bX_d)
=
\sum_{i\le t_d} w_{d,0}(X_{d,i})\Delta_{d,i}
-
\sum_{i\le \hat\xi_{t_d}(\bX_d)} w_d^{(t_d)}(X_{d,i})\Delta_{d,i},
\end{equation}
with $w_d^{(t_d)}(X_{d,i}) = w_{d,0}(X_{d,i})$ for $i\le t_d$ and $w_d^{(t_d)}(X_{d,i}) = w_{d,1}(X_{d,i})$ for $i>t_d$, and $\hat\xi_{t_d}(\bX_d)$ defined as in~\eqref{eq:weighted_mle}, applied to stream $d$.
\end{definition}

The weights in Definition~\ref{def:weighted_croc} can be constructed using the hard or soft uncertainty-based weighting rules introduced in Sec.~\ref{sec:weighting-construction}, or learned through the meta-learning procedure described in Sec.~\ref{subsec:meta_uncertainty_weighting}. 
When the uncertainty model is optimized via the meta-learning procedure, we refer to the resulting method as meta-learned weighted CROC (MW-CROC).
The corresponding differentiable training objective and optimization procedure are summarized in Algorithm~\ref{alg:meta_learning} in Appendix~\ref{app:algorithm}.

\begin{proposition}[Coverage of W-CROC]
\label{prop:wcroc-coverage}
Under Assumption~\ref{assump:multi_stream_split_exchangeability}, the W-CROC confidence set $\mathcal K_{\alpha}(\bX)$ in~\eqref{eq:croc_confidence_set}, computed with the weighted score~\eqref{eq:weighted_croc_score}, satisfies the marginal coverage condition
\begin{equation}
\Pr\bigl(d^\star\in\mathcal K_{\alpha}(\bX)\bigr)\ge 1-\alpha.    
\end{equation}
\end{proposition}

\begin{proof}
The weighted score~\eqref{eq:weighted_croc_score} is permutation-equivariant under the split-permutation group $\Pi_{\bt}$, since each stream's weights depend on $\bX_d$ only through stream-wise thresholds that are invariant under within-stream permutations, exactly as in the proof of Proposition~\ref{prop:weighted-coverage}. The coverage guarantee therefore follows directly from the validity proof of CROC~\cite{hore2026distribution}, since the weighted score preserves the required permutation-equivariance.    
\end{proof}

\subsection{Incorporating Prior Information on the Root Cause}
\label{sec:wcroc_prior}
When prior information is available on which streams are more likely to be the root cause $d^\star$, it can be incorporated via the level-allocation strategy described in Sec.~\ref{sec:wconch_prior}. The prior is encoded as weights $v_1,\dots,v_D\ge 0$ over the candidate streams, which are normalized as $\sum_{d=1}^{D}v_d=D$, so that values $v_d>1$ correspond to streams deemed more likely than under a uniform prior.

As in~\eqref{eq:prior_level}, each stream $d$ is tested at the level
\begin{equation}\label{eq:prior_level_rc}
    \alpha_d=\min\left(\frac{\alpha}{v_d},\,\alpha_{\max}\right),
\end{equation}
with $\alpha_{\max}\in[\alpha,1)$, yielding the prior-informed root-cause confidence set
\begin{equation}\label{eq:prior_confidence_set_rc}
    \mathcal K_{\alpha,v}(\bX)=\{d\in\{1,\dots,D\}: p_{(d)}>\alpha_d\},
\end{equation}
where the root-cause $p$-value $p_{(d)}$ in~\eqref{eq:croc_root_pvalue} is computed with the weighted score~\eqref{eq:weighted_croc_score}.

By~\eqref{eq:prior_level_rc}, a priori likely streams are tested at smaller levels, and are thus harder to exclude, while unlikely streams are more easily discarded. Under a uniform prior, i.e., $v_d=1$ for all $d$, all levels reduce to $\alpha_d=\alpha$, recovering the W-CROC set~\eqref{eq:croc_confidence_set}, and the role of $\alpha_{\max}$ is as discussed in Sec.~\ref{sec:wconch_prior}.

\begin{proposition}[Coverage of prior-informed W-CROC]
\label{prop:prior-coverage-rc}
Under Assumption~\ref{assump:multi_stream_split_exchangeability}, the set $\mathcal K_{\alpha,v}(\bX)$ in~\eqref{eq:prior_confidence_set_rc} satisfies
\begin{equation}
\Pr\bigl(d^\star\in\mathcal K_{\alpha,v}(\bX)\bigr)\ge 1-\alpha_{\max}  
\end{equation}
for any root cause $d^\star$. 
\end{proposition}
\begin{proof}
By the proof of Proposition~\ref{prop:wcroc-coverage}, under $d^\star=d$ the $p$-value $p_{(d)}$ is super-uniform, i.e., $\Pr\{p_{(d)}\le c\}\le c$ for any constant $c\in[0,1]$. Both bounds then follow as in Appendix~\ref{app:proof-prior-coverage}, with $t$, $\xi$, and $n-1$ replaced by $d$, $d^\star$, and $D$.    
\end{proof}

\section{Experiments}
\label{sec:experiments}

\begin{figure*}[t]
    \centering
    \includegraphics[width=\linewidth]{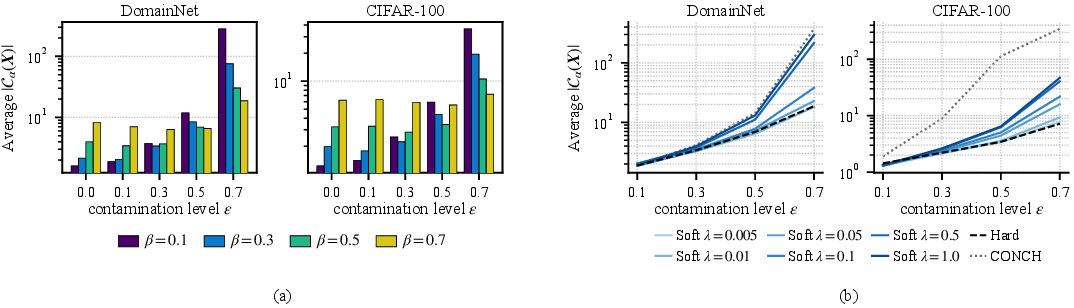}
    \caption{
    W-CONCH on DomainNet (left panels) and CIFAR-100 (right panels).
    (a): Average confidence set size under hard uncertainty thresholding~\eqref{eq:binary_weight_unc} for different quantile hyperparameters $\beta$.
    The results support the heuristic that hyperparameter $\beta$ should be chosen matched to the contamination level $\varepsilon$.
    (b): W-CONCH with soft weighting~\eqref{eq:soft_unc_weight} with $\beta=\varepsilon$ and different temperatures $\lambda$, compared with hard W-CONCH and CONCH~\cite{hore2026conformal}.
    }
    \label{fig:uncertainty_results}
\end{figure*}

We evaluate W-CONCH and W-CROC on image-based and real-world changepoint and root-cause localization benchmarks. For changepoint localization, we use DomainNet
(real-to-sketch domain shift), a CIFAR-100 benchmark (a shift between image classes), and the Milan telecom dataset~\cite{barlacchi2015multi} (a real shift in daily cellular activity at the holiday onset). 
For root-cause localization, we construct a multi-stream benchmark based on CIFAR-100, in which each stream undergoes a class shift at a stream-specific changepoint. In all settings, contamination is introduced by perturbing a random subset of observations according to the contamination model~\eqref{eq:contam_obs_model}.
Full implementation details, additional MNIST experiments, and complete numerical results are provided in Appendix~\ref{app:experiments}.

\subsection{Experimental Setup}
\label{subsec:experimental_setup}

\subsubsection{Datasets and Evaluation}\textbf{Changepoint localization.}
For DomainNet, sequences have length $n=800$ with changepoint $\xi=350$, where the first $\xi$ observations are from the real domain and the rest from the sketch domain. For CIFAR-100, sequences have length $n=400$ with $\xi=250$, shifting from class $99$ (worm) to class $77$ (snail). Observations are independently contaminated with probability $\varepsilon$ by Gaussian blur with $\sigma=20$ pixels for DomainNet and Gaussian noise with $\sigma=0.3$ pixels for CIFAR-100.

In the Milan telecom dataset, each sequence collects the daily activity profiles of one grid cell from call detail records recorded over Milan between November~2013 and January~2014, where each observation is a $720$-dimensional daily profile of $144$ ten-minute intervals across five call detail record measurements. Each sequence has $n=63$ samples with changepoint at $\xi=54$, corresponding to December~24, 2013. Observations are contaminated with probability $\varepsilon$ by Gaussian noise with $\sigma=20$ on standardized features.\\
\textbf{Root-cause localization.}
For CIFAR-100, we build a multi-stream benchmark with $D=5$ streams, each shifting between two classes within a common CIFAR-100 superclass (e.g., cattle to chimpanzee). Stream~1 is the root-cause stream and changes at $t_1=150$, while the remaining streams change at $t_d=155$ for $d=2,\ldots,5$. Each stream contains $n=400$ observations. The feasible set $\mathcal R$ contains one configuration per candidate root stream, assigning $t_d=150$ to that stream and $t_m=155$ to all others. Observations are contaminated as in the changepoint localization setup.\\
\textbf{Train and test pools.}
For all datasets, training and evaluation tasks are drawn from disjoint data pools, so that meta-learning never sees observations used at evaluation.\\
\textbf{Evaluation.}
We vary the contamination level over $\varepsilon \in \{0.0,0.1,0.3,0.5,0.7\}$. For changepoint localization, we report the average confidence set size $|\mathcal C_{\alpha}(\bX)|$ and empirical coverage over $200$ test tasks at $\alpha=0.05$, with $400$ split permutations per $p$-value. For root-cause localization, we evaluate $\mathcal K_{\alpha}(\bX)$ at $\alpha=0.01$ with $100$ split permutations. Since the MC approximation can leave all $D$ streams below $\alpha$, yielding an empty set, we report the penalized size $|\mathcal K_{\alpha}(\bX)|_{\mathrm{pen}}$, assigning size $D$ to empty realizations. When explicitly stated, we will also evaluate the prior-informed variant from Sec.~\ref{sec:wcroc_prior}, assigning weights $v_d=v>1$ to the $k$ most likely streams and $v_d'=({D-kv})/({D-k})$ for all other streams. The true root stream is always included among the $k$ likely streams, indicating a well specified prior.

\subsubsection{Baselines}
We compare the proposed methods against an unweighted baseline that ignores contamination and an oracle baseline with access to the true contamination of each observation:
\begin{itemize}
    \item CONCH and CROC~\cite{hore2026conformal,hore2026distribution}: The original methods that ignore contamination, as described in Sec.~\ref{subsec:conformal_changepoint_localization} and Sec.~\ref{subsec:croc}, respectively.
    \item Oracle-based W-CONCH and W-CROC: These methods use the true contamination indicators $Y_i$ in~\eqref{eq:contam_obs_model} to define oracle weights $w_i^{\mathrm{orc}}=1-Y_i$, assigning weight $1$ to clean observations and weight $0$ to contaminated observations.
\end{itemize}

\begin{figure*}[t]
    \centering
    \includegraphics[width=\linewidth]{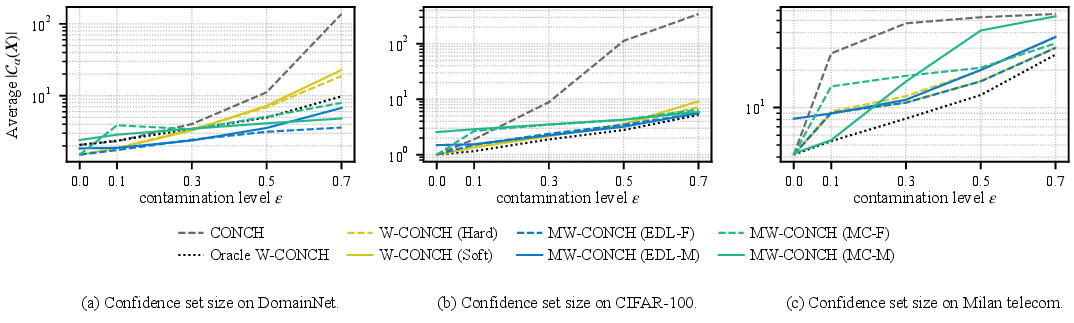}
    \caption{
    Performance against the contamination levels $\varepsilon$ for CONCH~\cite{hore2026conformal}, oracle-based W-CONCH, W-CONCH with hard and soft weighting, and MW-CONCH using EDL or MC dropout uncertainty models (EDL-F and MC-F denote fixed-$\varepsilon_{\mathrm{tr}}$ meta-learning with $\beta=\varepsilon$, while EDL-M and MC-M denote mixed-$\varepsilon_{\mathrm{tr}}$ meta-learning with $\beta=0.3$):
    (a) Average confidence set size $|\mathcal C_{\alpha}(\bX)|$ on DomainNet.
    (b) Average confidence set size $|\mathcal C_{\alpha}(\bX)|$ on CIFAR-100.
    (c) Average confidence set size $|\mathcal C_{\alpha}(\bX)|$ on the Milan telecom dataset.
    }
    \label{fig:overall_results}
\end{figure*}

\begin{figure}[t]
    \centering
    \includegraphics[width=\linewidth]{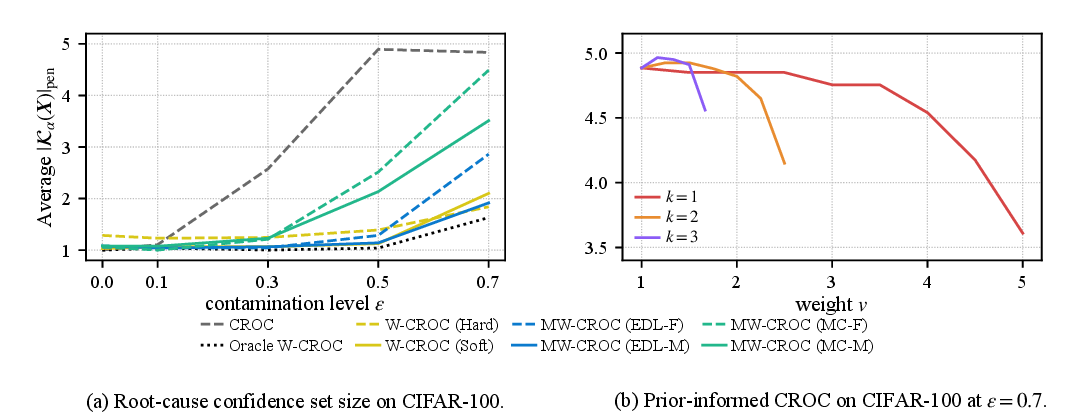}
    \caption{
        W-CROC and prior-informed CROC on CIFAR-100 with $D=5$ streams, $\sigma=0.3$, $n=400$, and $\alpha=0.01$.
        (a) Average penalized confidence set size $|\mathcal K_{\alpha}(\bX)|_{\mathrm{pen}}$ across contamination levels.
        (b) Effect of the prior weight $v$ on the penalized set size at $\varepsilon=0.7$ with $\alpha_{\max}=0.15$, where $k$ denotes the number of streams deemed likely by the prior.
        }
    \label{fig:wcroc_results}
\end{figure}
\subsection{Results}
\label{subsec:results}
\subsubsection{Changepoint Localization} \textbf{Effect of the weighting parameters.}
We first examine how the quantile hyperparameter $\beta$ and the soft-weighting temperature hyperparameter $\lambda$ affect efficiency. 
Fig.~\ref{fig:uncertainty_results}(a) shows the confidence set size obtained by hard uncertainty thresholding for different values of $\beta$ on DomainNet and CIFAR-100. The results support the heuristic that, when $\varepsilon$ is known, hyperparameter $\beta$ should be matched to $\varepsilon$.

Setting $\beta=\varepsilon$, Fig.~\ref{fig:uncertainty_results}(b) compares hard weighting with soft weighting for different hyperparameter $\lambda$. 
Based on these results, we use $\lambda=0.05$ for DomainNet and $\lambda=0.01$ for CIFAR-100 in the W-CONCH comparisons below.\\
\textbf{Comparison against contamination levels.}
We compare CONCH, oracle-based W-CONCH, and the proposed W-CONCH variants across contamination levels. 
For hard and soft W-CONCH, we use the matched choice $\beta=\varepsilon$.
Following the meta-training design discussed in Sec.~\ref{subsec:meta_uncertainty_weighting}, we evaluate MW-CONCH under both fixed-$\varepsilon_{\mathrm{tr}}$ meta-learning, tested with $\varepsilon_{\mathrm{tr}}=\varepsilon$, and mixed-$\varepsilon_{\mathrm{tr}}$ meta-learning. We consider the EDL and MC dropout uncertainty models, denoting the resulting variants MW-CONCH (EDL-F/MC-F) and MW-CONCH (EDL-M/MC-M), respectively.

As shown in Fig.~\ref{fig:overall_results}(a) and (b), CONCH becomes increasingly inefficient as contamination grows, with the confidence set size on DomainNet rising from $2.07$ at $\varepsilon=0$ to $137.43$ at $\varepsilon=0.7$, and similarly on CIFAR-100 from $1.00$ to $345.27$.
Uncertainty-based weighting reduces this inflation. At $\varepsilon=0.7$, hard and soft W-CONCH bring the set size down to $18.77$ and $22.84$ on DomainNet, and to $7.24$ and $9.18$ on CIFAR-100.

Meta-learning improves efficiency further.
On DomainNet, at $\varepsilon=0.7$, MW-CONCH (EDL-F) reaches $3.61$, even outperforming oracle-based W-CONCH, which reaches $9.81$, suggesting that meta-learning not only improves the uncertainty estimates used for weighting, but also learns classifier representations that are more robust to contamination. MW-CONCH (EDL-M) produces a slightly larger confidence set at $\varepsilon=0.7$, reaching $6.76$, but still substantially improves over direct weighting, while not requiring knowledge of the test contamination level.
On CIFAR-100, at $\varepsilon=0.7$, MW-CONCH (EDL-F) reaches $5.59$, close to the oracle-based W-CONCH at $5.21$. MW-CONCH (EDL-M) reaches $5.64$, again substantially improving over W-CONCH.

Fig.~\ref{fig:overall_results}(c) reports the corresponding results on the Milan telecom dataset, where the same qualitative trends hold.
CONCH inflates from $4.18$ at $\varepsilon=0$ to $56.37$ at $\varepsilon=0.7$, while hard and soft W-CONCH substantially reduce this inflation, reaching $36.54$ and $30.28$ at $\varepsilon=0.7$, respectively.
MW-CONCH (EDL-F) performs similarly to soft weighting, reaching $30.14$ at $\varepsilon=0.7$, while MW-CONCH (EDL-M) achieves a comparable reduction, reaching $36.91$ at $\varepsilon=0.7$.
As a trade-off, MW-CONCH (EDL-M) is less efficient at low contamination, reaching $8.14$ at $\varepsilon=0$ compared to $4.18$ for CONCH.
Oracle-based W-CONCH remains the most efficient throughout, reaching $26.62$ at $\varepsilon=0.7$.

The MC dropout variants also reduce the set size at high contamination, showing that the framework is not tied to EDL, although they are less effective than their EDL counterparts.
This suggests that the efficiency gains depend partly on the quality of the uncertainty estimates used for weighting.
Overall, uncertainty-based weighting substantially improves the efficiency of CONCH under contamination, while meta-learning further improves both the uncertainty estimates used for weighting and the robustness of the learned classifier under contamination.

\subsubsection{Root-Cause Localization}
Figure~\ref{fig:wcroc_results}(a) reports the average penalized root-cause confidence set size $|\mathcal K_{\alpha}(\bX)|_{\mathrm{pen}}$ across contamination levels.
As contamination increases, the set size produced by CROC grows substantially, reaching $4.835$ at $\varepsilon=0.7$, indicating a loss of root-cause specificity.
Uncertainty-based weighting reduces this inflation at high contamination: at $\varepsilon=0.7$, W-CROC with hard and soft weighting reduces the penalized set size to $1.840$ and $2.100$, respectively.
Oracle-based W-CROC further reduces the penalized set size to $1.630$, providing an upper performance bound.
MW-CROC (EDL-M) achieves a similar reduction, with penalized set size $1.915$, while not requiring knowledge of the test contamination level.
MW-CROC (EDL-F) also improves over CROC, but gives a larger set size of $2.860$ at $\varepsilon=0.7$.
The MC dropout variants reduce the set size in some regimes but are generally less effective than their EDL counterparts at high contamination, with MW-CROC (MC-F) and MW-CROC (MC-M) reaching $4.490$ and $3.510$, respectively. This suggests that the efficiency gains depend partly on the quality of the uncertainty estimates used for weighting.

Figure~\ref{fig:wcroc_results}(b) evaluates the prior-informed CROC described in Sec.~\ref{sec:wcroc_prior} at $\varepsilon=0.7$ with $\alpha_{\max}=0.15$. As the prior weight $v$ increases from $1$ and $k$ decreases, the prior grows more informative and, as a result, the penalized set size decreases. 
With $k=1$, the set size drops from $4.885$ at $v=1$ to $3.610$ at $v=5$. Larger $k$ produces more moderate reductions, reaching $4.150$ at $v=2.5$ for $k=2$ and $4.560$ at $v\approx1.67$ for $k=3$, because fewer unlikely streams remain to absorb the redistributed significance budget.

\section{Conclusions}\label{sec:conclusions}

We studied conformal changepoint localization and root cause analysis under contaminated observations. We showed that split exchangeability, and therefore the finite-sample distribution-free coverage guarantees of CONCH~\cite{hore2026conformal} and CROC~\cite{hore2026distribution}, are preserved under a Huber-type contamination model, while contamination can substantially inflate confidence set size.
To address this, we proposed W-CONCH and W-CROC, which downweight observations that are likely to be corrupted without affecting conformal calibration, thereby preserving coverage while reducing confidence set size. We introduced a meta-learning approach that learns the weighting rule without requiring knowledge of the contamination level at test time. Experiments on image-based and real-world changepoint and root-cause benchmarks demonstrated substantial reductions in confidence set size while maintaining the target coverage.

Several research directions remain open. The framework connects confidence sets to downstream, potentially risk-averse, decision making, and coupling the size objective directly to a decision loss is a natural next step. Extending the approach to streaming and online settings, to multiple or unknown numbers of changepoints, and to richer inter-stream dependency structures in root cause analysis would broaden its applicability. Finally, characterizing how the informativeness of the confidence set degrades as a function of the contamination level, and deriving guarantees on set size under structured contamination, would complement the coverage guarantees established here \cite{zecchin2024generalization}.

\appendices

\section{Proofs and Derivations}
\label{app:proofs}

\subsection{Proof of Lemma~\ref{lem:coverage-exchangeability}}
\label{app:proof-exchangeability-contamination}
We prove the claim for the pre-change segment. 
The post-change case follows identically with indices shifted to $\{\xi+1,\ldots,n\}$.
Fix a permutation $\pi_{0,\xi}$ of $\{1,\ldots,\xi\}$.
Since $(Y_1,Z_1),\ldots,(Y_\xi,Z_\xi)$ are i.i.d.\ and independent of the
clean segment, Assumption~\ref{asm:split-exch} implies that the augmented
sequence $\big\{(\widetilde X_i,Y_i,Z_i)\big\}_{i\le\xi}$ is exchangeable.
Applying the deterministic contamination map
$\phi(\widetilde x,y,z)=(1-y)\widetilde x+yz$ componentwise, which satisfies
$X_i=\phi(\widetilde X_i,Y_i,Z_i)$, preserves this distributional equality $(X_1,\ldots,X_\xi)
\overset{d}{=}
(X_{\pi_{0,\xi}(1)},\ldots,X_{\pi_{0,\xi}(\xi)}).$
Since $\pi_{0,\xi}$ was arbitrary, $(X_1,\ldots,X_\xi)$ is exchangeable. 
The same argument for $\{\xi+1,\ldots,n\}$ establishes split exchangeability of $\mathbf X$ at $\xi$.

\subsection{Proof of Proposition~\ref{prop:elbo_bound}}
\label{app:elbo_bound}
We lower-bound the mixture marginal $g_j(x)$ in~\eqref{eq:contam_marginal} by introducing the latent contamination indicator $Y\in\{0,1\}$ and applying the ELBO to $\log g_j(x)$.
Under the contamination model, $Y$ has prior $p(Y=0)=1-\varepsilon$ and $p(Y=1)=\varepsilon$, with conditional densities $g_j(x| Y=0)=f_j(x)$ and $g_j(x|Y=1)=q(x)$, so that $g_j(x)=\sum_{y}g_j(x|y)p(y)$ as in~\eqref{eq:contam_marginal}.

For any auxiliary posterior $r_j(y\mid x)$ on $\{0,1\}$, the ELBO identity~\cite{simeone2022machine} gives
\begin{align}\label{eq:elbo_drop}
    \log g_j(x)
    &= \mathbb E_{r_j(y\mid x)}\big[\log g_j(x\mid y)\big] - \operatorname{KL}\big(r_j(y\mid x)\|p(y)\big) + \operatorname{KL}\big(r_j(y\mid x)\|g_j(y\mid x)\big) \nonumber\\
    &\ge \mathbb E_{r_j(y\mid x)}\big[\log g_j(x\mid y)\big] - \operatorname{KL}\big(r_j(y\mid x)\|p(y)\big),
\end{align}
where the inequality drops the nonnegative divergence $\operatorname{KL}\big(r_j(y\mid x)\|g_j(y\mid x)\big)$.

Set the clean-posterior weight $w_j(x):=r_j(Y=0|x)$, so that $r_j(Y=1| x)=1-w_j(x)$. Since $Y\in\{0,1\}$, the first term of~\eqref{eq:elbo_drop} satisfies
\begin{align}\label{eq:first_term}
    \mathbb E_{r_j(y\mid x)}\big[\log g_j(x\mid y)\big] &= w_j(x)\log f_j(x) + \big(1-w_j(x)\big)\log q(x) \nonumber\\
    &\ge w_j(x)\log f_j(x) - C_q,
\end{align}
using $g_j(x\mid Y{=}0)=f_j(x)$, $g_j(x\mid Y{=}1)=q(x)$, and $\log q(x)\ge -C_q$ with $1-w_j(x)\in[0,1]$. For the second term, since $r_j(y\mid x)\le 1$ and $p(y)\ge\min\{\varepsilon,1-\varepsilon\}$ for $\varepsilon\in(0,1)$,
\begin{align}
    \operatorname{KL}\big(r_j(y\mid x)\|p(y)\big)&=\mathbb E_{r_j(y\mid x)}\Big[\log\tfrac{r_j(y\mid x)}{p(y)}\Big]\nonumber \\
    &\le \log\tfrac{1}{\min\{\varepsilon,1-\varepsilon\}}.
\end{align}\label{eq:kl_bound}
Substituting~\eqref{eq:first_term} and~\eqref{eq:kl_bound} into~\eqref{eq:elbo_drop} gives $\log g_j(x)\ge w_j(x)\log f_j(x)-C_R$ with $C_R=C_q+\log(1/\min\{\varepsilon,1-\varepsilon\})$, which is independent of $j$. Setting $C:=e^{-C_R}$, which depends only on $C_q$ and $\varepsilon$, and exponentiating yields~\eqref{eq:elbo_bound}.

\subsection{Derivation of the Classifier-Based Weighted CPP Score}
\label{app:classifier-cusum-derivation}
Using the classifier log-posterior ratio $\Delta_i$ in~\eqref{eq:classifier_log_posterior_contrast}, we substitute this approximation into the weighted CPP score~\eqref{eq:weighted_cpp_score}. 
We use the candidate-dependent weight $w^{(t)}$, the weighted log-likelihood $N_t(s)$, and the weighted maximum-likelihood estimate $\hat\xi_t(\bX)=\arg\max_s N_t(s)$ from Definition~\ref{def:clf_weighted_cpp}, so that the numerator and denominator of~\eqref{eq:weighted_cpp_score} are $N_t(t)$ and $N_t(\hat\xi_t(\bX))$, respectively.

Applying the identity $\log\hat p(j=0\mid X_i)=\Delta_i+\log\hat p(j=1\mid X_i)$ to the pre-change sum in the definition of $N_t(s)$ gives, for any cutoff $s$,
\begin{equation}\label{eq:cusum_num}
N_t(s)=\sum_{i\le s} w^{(t)}(X_i)\Delta_i + C_1,
\end{equation}
where the remainder $C_1=\sum_{i} w^{(t)}(X_i)\log\hat p(j=1\mid X_i)$ does not depend on $s$, and $w^{(t)}(X_i)=w_0(X_i)$ for $i\le t$ and $w_1(X_i)$ for $i>t$.
Evaluating~\eqref{eq:cusum_num} at $s=t$ for the numerator and at $s=\hat\xi_t(\bX)$ for the denominator, and subtracting, cancels the common remainder $C_1$ and recovers the classifier-based weighted CPP score in Definition~\ref{def:clf_weighted_cpp}.

\subsection{Proof of Proposition~\ref{prop:weighted-coverage}}
\label{app:proof-weighted-coverage}
Fix $t\in\{1,\dots,n-1\}$. Under the null $H_0:\xi=t$, Lemma~\ref{lem:coverage-exchangeability} gives $\pi(\bX)\overset{d}{=}\bX$ for all $\pi\in\Pi_t$. We show that $p_t^w$ in~\eqref{eq:wconch_confidence_set} is super-uniform. 
Taking $t=\xi$, where $H_0$ holds, then gives $\Pr(\xi\notin\mathcal C_\alpha^w)=\Pr(p_\xi^w\le\alpha)\le\alpha$, which is~\eqref{eq:weight_coverage}.

The weights depend on $\bX$ only through the thresholds $\kappa_0(t,\beta),\kappa_1(t,\beta)$ in~\eqref{eq:side_quantiles}, which are empirical quantiles of $\{M_i\}_{i\le t}$ and $\{M_i\}_{i>t}$. A permutation $\pi\in\Pi_t$ acts only within each segment and leaves these multisets unchanged, so the thresholds are invariant under $\pi$. Since $M_i=h(X_i)$ depends only on $X_i$, the weights are permuted consistently with the observations:
\begin{equation}\label{eq:perm_equiv_weights}
    w_j(\pi(\bX))=\pi\bigl(w_j(\bX)\bigr),\qquad \forall \pi\in\Pi_t,\quad j=0,1.
\end{equation}

Thus the weighted CPP score $S_t^w$ is permutation-equivariant. Averaging $\Ind \{p_t^w(\pi(\bX)) \le \alpha\}$ over $\pi\in\Pi_t$ using $\pi(\bX)\overset{d}{=}\bX$, expanding the definition of $p_t^w$ in~\eqref{eq:wconch_confidence_set}, and relabeling $\pi'\circ\pi$ as $\pi'$ via the group property of $\Pi_t$ gives
\begin{equation}\label{eq:cov_final}
\Pr\bigl(p_t^w(\bX)\le\alpha\bigr)
=\mathbb E\Biggl[\frac{1}{|\Pi_t|}\sum_{\pi\in\Pi_t}\Ind\biggl\{\frac{1}{|\Pi_t|}
\sum_{\pi'\in\Pi_t}\Ind\bigl[S_t^w(\pi'(\bX))
\le S_t^w(\pi(\bX))\bigr]\le\alpha\biggr\}\Biggr] \le\alpha,
\end{equation}
where the inner average is the rank of $S_t^w(\pi(\bX))$ among $\{S_t^w(\pi'(\bX))\}_{\pi'\in\Pi_t}$, and the final inequality is the finite-sample rank bound of~\cite[Thm.~3.1]{hore2026conformal}. Hence $p_t^w$ is super-uniform under $H_0$.

\subsection{Proof of Proposition~\ref{prop:prior-coverage}}
\label{app:proof-prior-coverage}
The argument in Appendix~\ref{app:proof-weighted-coverage} shows that, under the null $H_0:\xi=t$, the $p$-value $p_t^w(\bX)$ is super-uniform, i.e., $\Pr\bigl(p_t^w(\bX)\le c\bigr)\le c$ for any constant $c\in[0,1]$, since the rank bound~\eqref{eq:cov_final} does not depend on the specific value of the threshold.
To establish the bound $\Pr\bigl(\xi\in\mathcal C_{\alpha,v}^{w}(\bX)\bigr)\ge 1-\alpha_{\max}$, the changepoint $\xi$ is excluded from the set~\eqref{eq:prior_confidence_set} only if $p_\xi^w(\bX)\le\alpha_\xi$. Applying the super-uniformity at $t=\xi$ with $c=\alpha_\xi$, and using $\alpha_\xi\le\alpha_{\max}$ from~\eqref{eq:prior_level}, gives
\begin{equation}
    \Pr\bigl(\xi\notin\mathcal C_{\alpha,v}^{w}(\bX)\bigr)
    =\Pr\bigl(p_\xi^w(\bX)\le\alpha_\xi\bigr)
    \le\alpha_\xi
    \le\alpha_{\max}.
\end{equation}
For the prior-averaged guarantee~\eqref{eq:prior_avg_coverage}, averaging over the prior $\Pr(\xi=t)=\frac{v_t}{n-1}$ with the bound $\alpha_t\le\frac{\alpha}{v_t}$ from~\eqref{eq:prior_level} gives
\begin{equation}
    \Pr\bigl(\xi\notin\mathcal C_{\alpha,v}^{w}(\bX)\bigr)
    =\sum_{t=1}^{n-1}\frac{v_t\cdot\Pr\bigl(p_t^w(\bX)\le\alpha_t\mid\xi=t\bigr)}{n-1} \le \alpha
\end{equation}

\begin{algorithm}[t]
\caption{MW-CONCH and MW-CROC}
\label{alg:meta_learning}
\begin{algorithmic}[1]
\Require 
Training contamination levels $\{\varepsilon_{\mathrm{tr}}^{(\ell)}\}_{\ell=1}^{L}$, with $L=1$ for fixed-$\varepsilon_{\mathrm{tr}}$ training and $L>1$ for mixed-$\varepsilon_{\mathrm{tr}}$ training; training tasks $\{\bX^{(k,\ell)}\}_{k=1}^{K}$ at each $\varepsilon_{\mathrm{tr}}^{(\ell)}$, where each task is a single sequence for MW-CONCH and a $D$-stream tuple $\bX^{(k,\ell)}=(\bX_1^{(k,\ell)},\ldots,\bX_D^{(k,\ell)})$ for MW-CROC; learning rate $\eta$; quantile parameter $\beta$; smoothing parameters $\tau_q,\tau_1,\tau_2$; soft-weight temperature $\lambda$
\State Initialize $\theta$
\For{each gradient step}
  \For{each level $\varepsilon_{\mathrm{tr}}^{(\ell)}$ and each single-stream sequence $Z$ in its training tasks (i.e., $Z=\bX^{(k,\ell)}$ for MW-CONCH, or $Z=\bX_d^{(k,\ell)}$, $d=1,\dots,D$, for MW-CROC)}
    \State Compute uncertainty scores $M_i = h(Z_i;\theta)$ for $i=1,\dots,n$
    \For{Candidate changepoints $t=1,\dots,n-1$}
      \State Compute $\hat\kappa_0(t,\beta)$ and $\hat\kappa_1(t,\beta)$ using~\eqref{eq:soft-quantile}
      \State Compute $\hat w_{0,i}(t;\theta)$ and $\hat w_{1,i}(t;\theta)$ using~\eqref{eq:smooth-weight}
      \State Compute $\tilde p_t(Z;\theta)$ using~\eqref{eq:smooth-pvalue}
    \EndFor
    \State Compute $\tilde\ell(Z;\theta)$ from $\{\tilde p_t(Z;\theta)\}_{t=1}^{n-1}$ using~\eqref{eq:smooth-loss-pertask}
  \EndFor
  \State For each $\varepsilon_{\mathrm{tr}}^{(\ell)}$, compute $\tilde{\mathcal L}(\theta;\varepsilon_{\mathrm{tr}}^{(\ell)})$ by averaging $\tilde\ell(Z;\theta)$ over all sequences at that level ($\frac{1}{K}\sum_{k}$ for MW-CONCH, $\frac{1}{KD}\sum_{d}\sum_{k}$ for MW-CROC)
  \State Update $\theta \leftarrow \theta - \eta \nabla_\theta \left[\frac{1}{L}\sum_{\ell=1}^{L} \tilde{\mathcal L}(\theta;\varepsilon_{\mathrm{tr}}^{(\ell)})\right]$
\EndFor
\State \Return learned uncertainty estimator $h(\cdot;\theta^*)$
\end{algorithmic}
\end{algorithm}
\subsection{Differentiable Training Loss and Meta-learning Algorithm for MW-CONCH and MW-CROC}
\label{app:algorithm}
The meta-training loss~\eqref{eq:meta-loss} is not differentiable because of three operations: the empirical quantiles in the thresholds~\eqref{eq:side_quantiles}, the indicator in the split-permutation $p$-value, and the counting operation used to compute the confidence set size~\eqref{eq:wconch_confidence_set}. 
We replace each by a smooth surrogate controlled by a temperature parameter, following~\cite{park2023few,stutz2021learning}.

First, we replace the empirical quantile in~\eqref{eq:side_quantiles} by a differentiable counterpart~\cite{park2023few}. For any finite set $\{a_r\}_{r=1}^m$ and quantile level $1-\beta$, let
\begin{equation}\label{eq:pinball}
\rho_{1-\beta}(a;\{a_r\}_{r=1}^{m+1})
=\sum_{r=1}^{m+1} \beta(a-a_r)^++(1-\beta)(a_r-a)^+,
\end{equation}
denote the pinball loss, where $(x)^+=\max(x,0)$ and $a_{m+1}=\max(\{a_r\}_{r=1}^m)+\delta$ with $\delta>0$ is an auxiliary point ensuring the soft quantile is well-defined at the boundary. The empirical quantile is then approximated by
\begin{equation}\label{eq:soft-quantile}
\hat Q_{1-\beta}(\{a_r\}_{r=1}^m)
=
\sum_{r=1}^{m+1}
a_r\,
\mathrm{softmax}\bigl(\frac{-\rho_{1-\beta}(a_r)}{\tau_q}\bigr),
\end{equation}
where $\mathrm{softmax}(x_r) = {\exp(x_r)}/{\sum_{\ell}\exp(x_\ell)}$ and $\tau_q>0$ is a smoothing temperature. Replacing the empirical quantiles in~\eqref{eq:side_quantiles} by $\hat Q_{1-\beta}$ yields smoothed thresholds $\hat \kappa_j(t,\beta)$, with which we define the differentiable training-time weight
\begin{equation}\label{eq:smooth-weight}
    \hat w_j(X_i;\theta)
    =
    \sigma\left(\frac{\hat \kappa_j(t,\beta)-h(X_i;\theta)}{\lambda}\right), j\in\{0,1\},
\end{equation}
in direct analogy with the soft weighting rule~\eqref{eq:soft_unc_weight}.

Second, we replace the indicator in the split-permutation $p$-value by a sigmoid. 
Writing $S_t^{\hat w}$ for the weighted CPP score evaluated with the training-time weights $\{\hat w_j(\cdot;\theta)\}_{j=0}^1$ in~\eqref{eq:smooth-weight}:
\begin{equation}\label{eq:smooth-pvalue}
\tilde p_t(\theta)
=
\frac{1}{|\Pi_t|}
\sum_{\pi\in\Pi_t}
\sigma\biggl(
\frac{1}{\tau_1}
\Bigl(
S_t^{\hat w}(\bX)
-
S_t^{\hat w}(\pi(\bX))
\Bigr)
\biggr),
\end{equation}
$\tau_1>0$ controls the sharpness of the sigmoid.

Third, we replace the threshold indicator in the set size by a sigmoid~\cite{stutz2021learning}, yielding the per-task smoothed size
\begin{equation}\label{eq:smooth-loss-pertask}
    \tilde\ell(\bX;\theta)
    =
    \sum_{t=1}^{n-1}
    \sigma\left(
    \frac{\tilde p_t(\theta)-\alpha}{\tau_2}
    \right),
\end{equation}
where $\tau_2>0$ controls the sharpness of the threshold.

Substituting $\tilde\ell$ for $|\mathcal C_{\alpha,w(\theta)}|$ in~\eqref{eq:meta-loss} gives the differentiable training loss $\tilde{\mathcal L}(\theta;\varepsilon)=\frac{1}{K}\sum_{k=1}^K\tilde\ell\bigl(\bX^{(k,\varepsilon)};\theta\bigr)$.
As the temperatures $\tau_q,\tau_1,\tau_2\to0$ are annealed during training, each surrogate recovers its non-differentiable counterpart and $\tilde{\mathcal L}(\theta;\varepsilon)\to\mathcal L(\theta;\varepsilon)$. The full procedure is given in Algorithm~\ref{alg:meta_learning}.

\section{Experimental Details and Results}
\label{app:experiments}

\subsection{Implementation Details}
\label{app:implementation_details}
\textbf{Training details and experimental settings.}
For DomainNet~\cite{peng2019moment} and CIFAR-100~\cite{krizhevsky2009learning}, the uncertainty model uses a ResNet-18~\cite{he2016deep} backbone. The same uncertainty model is used for both MW-CONCH and MW-CROC. During meta-learning fine-tuning, the backbone is frozen and only the selection head is updated. Unless otherwise specified, models are trained for $100$ epochs.\\
\textbf{MW-CONCH / MW-CROC (EDL-F).}
A separate model is trained for each fixed training contamination level
$\varepsilon_{\mathrm{tr}}\in\{0.1,0.3,0.5,0.7\}$, with the
soft-quantile threshold parameter set to $\beta=\varepsilon_{\mathrm{tr}}$.
At each epoch, $10$ tasks are sampled at contamination level
$\varepsilon_{\mathrm{tr}}$.\\
\textbf{MW-CONCH / MW-CROC (EDL-M).}
A single model is trained over a mixture of contamination levels, where
$\varepsilon_{\mathrm{tr}}$ is chosen from the grid
$\{0.0,0.1,\ldots,0.7\}$.
The threshold parameter is fixed at $\beta=0.3$ for all mixed-$\varepsilon_{\mathrm{tr}}$ experiments.
At each epoch, $5$ tasks are sampled per level, yielding $40$ tasks in total.\\
\textbf{Changepoint position randomization.}
To prevent the model from learning position-specific suppression patterns,
the changepoint position $\xi$ is sampled uniformly from a fixed range
for each training task rather than being held fixed. For DomainNet ($n=800$),$\xi\sim\mathrm{Unif}\{160,\dots,640\}.$
For CIFAR-100 ($n=400$), $\xi\sim\mathrm{Unif}\{120,\dots,280\}$ in the W-CONCH setting and $\xi\sim\mathrm{Unif}\{80,\dots,320\}$ in the W-CROC setting.

For both training regimes and all datasets, the smoothing parameters $\tau_1$ and $\tau_2$ are initialized at $0.5$ and annealed with decay factor $0.95$ per epoch. The soft-quantile temperature is fixed at $\tau_q=1.0$, and the same soft-quantile operator is used to compute the side-specific thresholds of the soft weighting rule at evaluation time.
The sigmoid temperature $\lambda$ for the soft weights is set to $\lambda=0.05$ for all DomainNet experiments and $\lambda=0.01$ for CIFAR-100 W-CONCH experiments. For CIFAR-100 W-CROC experiments, we use $\lambda=0.01$ for soft weighting and EDL-based uncertainty, and use $\lambda=0.005$ only for the MW-CROC variants based on MC Dropout.

For DomainNet, the smoothing temperatures are annealed to a minimum of $\tau_{\min}=0.01$, using $50$ split permutations per training task.
For CIFAR-100, $\tau_{\min}=0.05$ is used to avoid overly sharp sigmoid approximations at the end of the annealing schedule, with $50$ split permutations per training task. The same training configuration applies to both MW-CONCH and MW-CROC.\\
\textbf{Milan telecom: data construction.}
Since the dataset provides no cell-type annotations, evaluation cells are selected from their traffic profiles.
For each cell, we compute two ratios over the pre-holiday period (November~4--December~20): the mean daily internet traffic on weekends divided by that on weekdays, and the mean daily traffic during the Christmas week (December~23--27) divided by the pre-holiday mean.
We retain cells whose weekend-to-weekday ratio exceeds $0.8$, excluding office-type cells whose strong weekly pattern would confound the holiday change, and whose Christmas-to-normal ratio exceeds $0.7$, excluding cells whose activity collapses over the holidays and therefore lacks a well-defined localized changepoint; cells with low overall traffic or fewer than $50$ observed days are also discarded.
The ground-truth changepoint is detected as the largest activity shift within December~20--26, keeping only cells whose changepoint falls on December~24 ($\xi=54$).
Accordingly, training cells are labeled pre-change ($j=0$) before December~24 and post-change ($j=1$) thereafter.\\
\textbf{Milan telecom: models and uncertainty.}
All models use a residual MLP backbone ($720\to256$, three residual blocks), with standardization statistics computed on the training-cell pool.
Since $\mathrm{relu}(\text{logit})$ evidence is overconfident on noisy tabular inputs, the EDL model computes evidence with a radial basis function head~\cite{van2020uncertainty}, $e_k=\gamma\exp(-\|\phi(x)-c_k\|^2/(2\sigma^2))$, which decays with the distance of the representation $\phi(x)$ from learnable class centroids $c_k$ (with scales $\gamma,\sigma$ learned in log space and initialized to $\gamma=e^{2}$, $\sigma=1$); the score $\Delta_i=\log e_0-\log e_1$ and the uncertainty $M_i=2/(e_0+e_1+2)$ are both produced by this model, trained on class-balanced data with an evidence-magnitude penalty of $0.01$.
For MC Dropout, $M_i$ is the standard deviation of $\Delta_i$ across $T=50$ stochastic passes, as predictive entropy saturates under large tabular noise; unlike the EDL model, its underlying classifier is trained on all available days rather than class-balanced data.\\
\textbf{Milan telecom: meta-learning.}
The meta-learning procedure follows the settings described above, except that the full model is fine-tuned for $30$ epochs with learning rate $10^{-4}$, using $50$ tasks per contamination level and $100$ split permutations per task, with $\tau_1$ and $\tau_2$ initialized at $1.0$; the sigmoid temperature is $\lambda=0.05$, and the mixed-$\varepsilon_{\mathrm{tr}}$ models use $\beta=0.1$ for EDL and $\beta=0.3$ for MC Dropout.

{
\subsection{Detailed Experiment Results}
\label{app:experiment_details}

\subsubsection{Detailed Changepoint Localization Results}
\label{app:changepoint_details}
We provide detailed numerical results corresponding to the experiments presented in the main text.
Each entry reports the average confidence set size $|\mathcal{C}_\alpha(\bX)|$, with empirical coverage shown in parentheses.
\begin{table}[H]
\centering
\caption{
Detailed numerical results for DomainNet, CIFAR-100, and the Milan telecom dataset.
Each entry reports the average confidence set size $|\mathcal{C}_\alpha(\bX)|$ (coverage).
Hard, Soft, and EDL-F (as well as MC-F) use $\beta=\varepsilon$ (matched to the test contamination level).
For $\varepsilon=0$, the corresponding CONCH result is reported since $\beta=0$ is undefined.
}
\label{tab:overall_results}
\small
\begin{tabular}{lccccc}
\toprule
Method
& $\varepsilon=0.0$
& $0.1$
& $0.3$
& $0.5$
& $0.7$ \\
\midrule
\multicolumn{6}{l}{\textbf{DomainNet}} \\
\midrule
CONCH
& 2.07 (.940)
& 2.33 (.960)
& 4.03 (.945)
& 11.19 (.955)
& 137.43 (.965) \\
Oracle W-CONCH
& 2.07 (.940)
& 2.37 (.955)
& 3.44 (.965)
& 4.96 (.960)
& 9.81 (.950) \\
W-CONCH (Hard)
& 1.51 (.955)
& 1.88 (.955)
& 3.40 (.935)
& 6.92 (.965)
& 18.77 (.945) \\
W-CONCH (Soft)
& 1.51 (.955)
& 1.86 (.955)
& 3.29 (.930)
& 7.24 (.970)
& 22.84 (.945) \\
MW-CONCH (EDL-F)
& 1.51 (.955)
& 1.75 (.945)
& 2.44 (.960)
& 3.15 (.970)
& 3.61 (.970) \\
MW-CONCH (EDL-M)
& 1.85 (.945)
& 1.88 (.940)
& 2.38 (.940)
& 3.56 (.975)
& 6.76 (.970) \\
MW-CONCH (MC-F)
& 1.51 (.955)
& 3.88 (.955)
& 3.42 (.965)
& 5.12 (.950)
& 7.96 (.950) \\
MW-CONCH (MC-M)
& 2.42 (.940)
& 2.87 (.955)
& 3.47 (.925)
& 4.14 (.965)
& 4.82 (.980) \\
\midrule
\multicolumn{6}{l}{\textbf{CIFAR-100}} \\
\midrule
CONCH
& 1.00 (.985)
& 1.90 (.970)
& 8.87 (.935)
& 113.14 (.955)
& 345.27 (.975) \\
Oracle W-CONCH
& 1.00 (.985)
& 1.16 (.990)
& 1.89 (.990)
& 2.79 (.990)
& 5.21 (.965) \\
W-CONCH (Hard)
& 1.00 (.985)
& 1.42 (.970)
& 2.23 (.945)
& 3.42 (.945)
& 7.24 (.965) \\
W-CONCH (Soft)
& 1.00 (.985)
& 1.35 (.960)
& 2.15 (.955)
& 3.56 (.935)
& 9.18 (.950) \\
MW-CONCH (EDL-F)
& 1.00 (.985)
& 1.52 (.955)
& 2.39 (.950)
& 3.47 (.965)
& 5.59 (.935) \\
MW-CONCH (EDL-M)
& 1.49 (.960)
& 1.54 (.940)
& 2.24 (.935)
& 3.20 (.975)
& 5.64 (.915) \\
MW-CONCH (MC-F)
& 1.00 (.985)
& 2.68 (.965)
& 3.48 (.965)
& 4.24 (.915)
& 6.54 (.950) \\
MW-CONCH (MC-M)
& 2.56 (.955)
& 2.88 (.980)
& 3.52 (.945)
& 4.29 (.940)
& 5.93 (.935) \\
\midrule
\multicolumn{6}{l}{\textbf{Milan telecom}} \\
\midrule
CONCH
& 4.18 (.940)
& 27.18 (.915)
& 47.44 (.905)
& 53.12 (.930)
& 56.37 (.935) \\
Oracle W-CONCH
& 4.18 (.940)
& 5.38 (.960)
& 8.16 (.970)
& 12.60 (.950)
& 26.62 (.963) \\
W-CONCH (Hard)
& 4.18 (.940)
& 9.30 (.910)
& 12.33 (.925)
& 19.86 (.960)
& 36.54 (.970) \\
W-CONCH (Soft)
& 4.18 (.940)
& 9.02 (.905)
& 11.01 (.915)
& 16.27 (.950)
& 30.28 (.935) \\
MW-CONCH (EDL-F)
& 4.18 (.940)
& 9.05 (.900)
& 10.93 (.905)
& 16.23 (.950)
& 30.14 (.935) \\
MW-CONCH (EDL-M)
& 8.14 (.895)
& 8.95 (.910)
& 11.55 (.905)
& 20.03 (.950)
& 36.91 (.950) \\
MW-CONCH (MC-F)
& 4.18 (.940)
& 14.77 (.945)
& 17.97 (.955)
& 20.87 (.965)
& 32.72 (.970) \\
MW-CONCH (MC-M)
& 4.27 (.930)
& 5.48 (.965)
& 16.18 (.960)
& 41.54 (.980)
& 54.03 (.920) \\
\bottomrule
\end{tabular}
\end{table}

\subsubsection{Detailed Root-Cause Localization Results}
\label{app:rootcause_details}

We provide detailed numerical results for the W-CROC root-cause localization experiments.
Each entry reports the average penalized confidence set size $|\mathcal K_{\alpha}(\bX)|_{\mathrm{pen}}$, empirical coverage, and the per-stream root-cause $p$-values $p_{(1)},\ldots,p_{(5)}$, where stream~1 is the true root cause.

\begin{table}[H]
\centering
\caption{
W-CROC results on CIFAR-100 (noise $\sigma{=}0.3$, $n{=}400$,
$\alpha{=}0.01$, $100$ split permutations, $200$ test tasks).
W-CROC Hard and Soft use $\beta{=}0.7$;
MW-CROC EDL-F and MC-F use $\beta{=}\varepsilon$;
MW-CROC EDL-M and MC-M use mixed-$\varepsilon_{\mathrm{tr}}$ training
with $\beta{=}0.3$.
}
\label{tab:wcroc_cifar100}
\footnotesize
\begin{tabular}{cl|cc|ccccc}
\toprule
$\varepsilon$ & Method
  & $|{\mathcal{K}}_{\alpha}(\bX)|_{\mathrm{pen}}$
  & Cov.
  & $p_{(1)}$ & $p_{(2)}$ & $p_{(3)}$ & $p_{(4)}$ & $p_{(5)}$ \\
\midrule
\multirow{8}{*}{$0$}
  & CROC             & 1.000 & 1.000 & 0.910 & 0.010 & 0.010 & 0.010 & 0.010 \\
  & Oracle W-CROC    & 1.000 & 1.000 & 0.910 & 0.010 & 0.010 & 0.010 & 0.010 \\
  & W-CROC (Hard)    & 1.285 & 1.000 & 0.985 & 0.011 & 0.020 & 0.010 & 0.048 \\
  & W-CROC (Soft)    & 1.030 & 0.995 & 0.831 & 0.010 & 0.010 & 0.010 & 0.010 \\
  & MW-CROC (EDL-F)  & 1.060 & 0.985 & 0.824 & 0.010 & 0.010 & 0.010 & 0.010 \\
  & MW-CROC (EDL-M)  & 1.080 & 0.980 & 0.552 & 0.010 & 0.010 & 0.010 & 0.010 \\
  & MW-CROC (MC-F)   & 1.060 & 0.985 & 0.824 & 0.010 & 0.010 & 0.010 & 0.010 \\
  & MW-CROC (MC-M)   & 1.080 & 0.980 & 0.687 & 0.010 & 0.010 & 0.010 & 0.010 \\
\midrule
\multirow{8}{*}{$0.1$}
  & CROC             & 1.105 & 0.975 & 0.633 & 0.010 & 0.010 & 0.010 & 0.010 \\
  & Oracle W-CROC    & 1.040 & 0.990 & 0.925 & 0.010 & 0.010 & 0.010 & 0.010 \\
  & W-CROC (Hard)    & 1.230 & 0.995 & 0.976 & 0.011 & 0.010 & 0.010 & 0.032 \\
  & W-CROC (Soft)    & 1.065 & 0.985 & 0.610 & 0.010 & 0.010 & 0.010 & 0.010 \\
  & MW-CROC (EDL-F)  & 1.040 & 0.990 & 0.595 & 0.010 & 0.010 & 0.010 & 0.010 \\
  & MW-CROC (EDL-M)  & 1.040 & 0.990 & 0.555 & 0.010 & 0.010 & 0.010 & 0.010 \\
  & MW-CROC (MC-F)   & 1.000 & 1.000 & 0.628 & 0.010 & 0.010 & 0.010 & 0.010 \\
  & MW-CROC (MC-M)   & 1.075 & 0.990 & 0.592 & 0.010 & 0.010 & 0.010 & 0.010 \\
\midrule
\multirow{8}{*}{$0.3$}
  & CROC             & 2.575 & 0.995 & 0.492 & 0.052 & 0.013 & 0.022 & 0.017 \\
  & Oracle W-CROC    & 1.000 & 1.000 & 0.942 & 0.010 & 0.010 & 0.010 & 0.010 \\
  & W-CROC (Hard)    & 1.245 & 1.000 & 0.976 & 0.016 & 0.016 & 0.010 & 0.032 \\
  & W-CROC (Soft)    & 1.065 & 0.990 & 0.508 & 0.010 & 0.010 & 0.010 & 0.010 \\
  & MW-CROC (EDL-F)  & 1.045 & 0.990 & 0.500 & 0.010 & 0.010 & 0.010 & 0.010 \\
  & MW-CROC (EDL-M)  & 1.065 & 0.990 & 0.521 & 0.010 & 0.010 & 0.010 & 0.010 \\
  & MW-CROC (MC-F)   & 1.210 & 0.980 & 0.543 & 0.012 & 0.010 & 0.011 & 0.011 \\
  & MW-CROC (MC-M)   & 1.230 & 0.985 & 0.515 & 0.011 & 0.010 & 0.011 & 0.011 \\
\midrule
\multirow{8}{*}{$0.5$}
  & CROC             & 4.895 & 0.985 & 0.537 & 0.409 & 0.313 & 0.339 & 0.324 \\
  & Oracle W-CROC    & 1.040 & 1.000 & 0.906 & 0.010 & 0.010 & 0.010 & 0.010 \\
  & W-CROC (Hard)    & 1.390 & 1.000 & 0.971 & 0.031 & 0.017 & 0.016 & 0.043 \\
  & W-CROC (Soft)    & 1.125 & 0.995 & 0.520 & 0.010 & 0.010 & 0.010 & 0.014 \\
  & MW-CROC (EDL-F)  & 1.285 & 0.995 & 0.518 & 0.014 & 0.010 & 0.011 & 0.021 \\
  & MW-CROC (EDL-M)  & 1.140 & 1.000 & 0.548 & 0.011 & 0.010 & 0.011 & 0.013 \\
  & MW-CROC (MC-F)   & 2.515 & 0.980 & 0.500 & 0.070 & 0.028 & 0.059 & 0.046 \\
  & MW-CROC (MC-M)   & 2.135 & 0.995 & 0.523 & 0.032 & 0.021 & 0.034 & 0.022 \\
\midrule
\multirow{8}{*}{$0.7$}
  & CROC             & 4.835 & 0.990 & 0.503 & 0.332 & 0.501 & 0.241 & 0.208 \\
  & Oracle W-CROC    & 1.630 & 0.990 & 0.897 & 0.048 & 0.027 & 0.038 & 0.048 \\
  & W-CROC (Hard)    & 1.840 & 0.995 & 0.589 & 0.043 & 0.033 & 0.039 & 0.052 \\
  & W-CROC (Soft)    & 2.100 & 0.995 & 0.515 & 0.048 & 0.023 & 0.031 & 0.047 \\
  & MW-CROC (EDL-F)  & 2.860 & 0.985 & 0.531 & 0.111 & 0.044 & 0.018 & 0.165 \\
  & MW-CROC (EDL-M)  & 1.915 & 1.000 & 0.479 & 0.030 & 0.013 & 0.017 & 0.047 \\
  & MW-CROC (MC-F)   & 4.490 & 0.995 & 0.511 & 0.280 & 0.176 & 0.214 & 0.279 \\
  & MW-CROC (MC-M)   & 3.510 & 0.995 & 0.497 & 0.110 & 0.037 & 0.073 & 0.058 \\
\bottomrule
\end{tabular}
\end{table}

\subsubsection{Prior-Informed CROC Results}
\label{app:prior_croc_details}

Table~\ref{tab:prior_croc} reports the prior-informed CROC results on CIFAR-100 at $\varepsilon=0.7$ with $\alpha_{\max}=0.15$. The prior assigns weight $v_d=v>1$ to each of the $k$ most likely streams and $v_d'=\frac{D-kv}{D-k}$ to all other streams. The per-stream significance levels are $\alpha_d = \min(\alpha / v_d,\, \alpha_{\max})$. The true root stream is always included among the $k$ likely streams, indicating a well-specified prior. As $v$ increases, the prior concentrates more weight on the likely streams, allowing the unlikely streams to be tested at higher levels and thus more easily excluded. With $k=1$, the penalized set size drops from $4.885$ at the uniform baseline ($v=1$) to $3.610$ at $v=D=5$. Larger likely sets produce more moderate reductions, as fewer unlikely streams remain to absorb the redistributed significance budget.

\begin{table}[H]
\centering
\caption{Prior-informed CROC on CIFAR-100 at $\varepsilon{=}0.7$
with $\alpha{=}0.01$ and $\alpha_{\max}{=}0.15$.
Each row reports the penalized set size $|\mathcal K_{\alpha}(\bX)|_{\mathrm{pen}}$
and empirical coverage.
The baseline row corresponds to standard CROC at $\alpha{=}0.01$ ($v{=}1$).}
\label{tab:prior_croc}
\begin{tabular}{cl|cc}
\toprule
$k$ & $v$
  & $|{\mathcal{K}}_{\alpha}|_{\mathrm{pen}}$
  & Cov. \\
\midrule
\multicolumn{2}{c|}{CROC ($v{=}1$)} & 4.885 & 0.985 \\
\midrule
\multirow{9}{*}{$1$}
  & 1.0 & 4.885 & 0.985 \\
  & 1.5 & 4.850 & 1.000 \\
  & 2.0 & 4.850 & 1.000 \\
  & 2.5 & 4.850 & 1.000 \\
  & 3.0 & 4.755 & 1.000 \\
  & 3.5 & 4.755 & 1.000 \\
  & 4.0 & 4.540 & 1.000 \\
  & 4.5 & 4.175 & 1.000 \\
  & 5.0 & 3.610 & 1.000 \\
\midrule
\multirow{7}{*}{$2$}
  & 1.00 & 4.885 & 0.985 \\
  & 1.25 & 4.925 & 1.000 \\
  & 1.50 & 4.925 & 1.000 \\
  & 1.75 & 4.880 & 1.000 \\
  & 2.00 & 4.820 & 1.000 \\
  & 2.25 & 4.650 & 1.000 \\
  & 2.50 & 4.150 & 1.000 \\
\midrule
\multirow{5}{*}{$3$}
  & 1.00 & 4.885 & 0.985 \\
  & 1.17 & 4.965 & 1.000 \\
  & 1.33 & 4.950 & 1.000 \\
  & 1.50 & 4.910 & 1.000 \\
  & 1.67 & 4.560 & 1.000 \\
\bottomrule
\end{tabular}
\end{table}

\subsection{MNIST Experiments}
\label{app:mnist}
We report MNIST~\cite{deng2012mnist} results for both W-CONCH and W-CROC, complementing the DomainNet and CIFAR-100 experiments presented.

\subsubsection{MNIST Changepoint Localization Results}
Table~\ref{tab:overall_results_mnist} summarizes the changepoint localization results on MNIST (digits $3{\to}5$, $n{=}400$, $\xi{=}250$, $\alpha{=}0.05$, horizontal-flip-and-blur contamination with $\sigma{=}1.5$). Each entry reports the average confidence set size $|\mathcal{C}_\alpha(\bX)|$, with empirical coverage shown in parentheses.

CONCH deteriorates sharply under strong contamination, increasing from $1.00$ at $\varepsilon{=}0$ to $89.62$ at $\varepsilon{=}0.7$. Direct uncertainty-based weighting via W-CONCH substantially reduces the confidence set size, reaching $6.88$ (Hard) and $6.93$ (Soft) at $\varepsilon{=}0.7$. Meta-learning yields further improvements. At $\varepsilon{=}0.7$, MW-CONCH (EDL-F) achieves a confidence set size of $4.60$, while MW-CONCH (EDL-M) reduces it further to $2.31$. The MC Dropout variants are similarly effective, with MW-CONCH (MC-F) and MW-CONCH (MC-M) achieving confidence set sizes of $2.46$ and $3.23$, respectively. Coverage remains close to the nominal level $1-\alpha=0.95$ across all contamination levels, indicating that the proposed weighting schemes preserve the conformal coverage guarantee while substantially improving efficiency.

\begin{table}[H]
\centering
\caption{
MNIST changepoint localization results ($n{=}400$, $\xi{=}250$,
$\alpha{=}0.05$, $400$ split permutations, $200$ test tasks).
Each entry reports $|\mathcal{C}_\alpha(\bX)|$ (coverage).
W-CONCH (Hard, Soft), and MW-CONCH (EDL-F) use $\beta{=}\varepsilon$;
$\varepsilon{=}0$ is filled with CONCH since $\beta{=}0$ is undefined.
MW-CONCH (EDL-M) and MW-CONCH (MC-M) are trained on mixed contamination levels with $\beta{=}0.3$.
}
\label{tab:overall_results_mnist}
\small
\begin{tabular}{lccccc}
\toprule
Method
& $\varepsilon=0.0$
& $0.1$
& $0.3$
& $0.5$
& $0.7$ \\
\midrule
CONCH
& 1.00 (1.000)
& 1.24 (.975)
& 2.89 (.950)
& 8.39 (.970)
& 89.62 (.955) \\
Oracle W-CONCH
& 1.00 (1.000)
& 1.16 (.995)
& 1.89 (.985)
& 2.81 (.990)
& 5.15 (.970) \\
W-CONCH (Hard)
& 1.00 (1.000)
& 1.21 (.980)
& 1.97 (.930)
& 3.27 (.975)
& 6.88 (.950) \\
W-CONCH (Soft)
& 1.00 (1.000)
& 1.18 (.980)
& 1.99 (.955)
& 3.08 (.935)
& 6.93 (.945) \\
MW-CONCH (EDL-F)
& 1.00 (1.000)
& 1.15 (.970)
& 1.90 (.940)
& 2.62 (.925)
& 4.60 (.930) \\
MW-CONCH (EDL-M)
& 1.02 (.965)
& 1.06 (.950)
& 1.42 (.920)
& 1.64 (.955)
& 2.31 (.955) \\
MW-CONCH (MC-F)
& 1.00 (1.000)
& 1.18 (1.000)
& 2.03 (.985)
& 2.98 (.955)
& 2.46 (.945) \\
MW-CONCH (MC-M)
& 1.03 (.995)
& 1.12 (.990)
& 1.53 (.975)
& 2.10 (.985)
& 3.23 (.950) \\
\bottomrule
\end{tabular}
\end{table}

\subsubsection{MNIST Root-Cause Localization Results}
We provide detailed numerical results for the W-CROC root-cause localization experiments on MNIST. Each entry reports the average confidence set size $|\mathcal{K}_{\alpha}(\bX)|$, empirical coverage, and the per-stream root-cause $p$-values $p_{(1)},\ldots,p_{(5)}$, where stream1 is the true root stream. A stream $k$ is included in the confidence set whenever $p_{(k)}>\alpha$.

The multi-stream MNIST benchmark consists of $D=5$ streams, corresponding to the binary digit shifts $(2{\to}5)$, $(1{\to}7)$, $(3{\to}8)$, $(0{\to}6)$, and $(4{\to}9)$. Stream1 is the root-cause stream with changepoint at $t_1=150$, while the remaining streams change at $t_d=152$ for $d=2,\ldots,5$. Each stream contains $n=400$ observations. Each observation is independently contaminated with probability $\varepsilon$ using a horizontal-flip-and-blur transformation with Gaussian blur standard deviation $\sigma=1.5$.

Table~\ref{tab:wcroc_mnist} reports the detailed results. Similar to the CIFAR-100 benchmark, the confidence set size produced by CROC increases substantially as the contamination level grows, reaching $4.885$ at $\varepsilon=0.7$. The Genie reference maintains confidence set sizes close to one across all contamination levels, indicating that root-cause localization remains feasible when contamination is properly handled. W-CROC (Hard) provides moderate improvements over CROC, while the meta-learned variants further improve localization performance. In particular, MW-CROC (EDL-M) achieves the smallest confidence set sizes at moderate contamination levels ($\varepsilon=0.3$ and $0.5$), whereas MW-CROC (EDL-F) remains more effective under severe contamination ($\varepsilon=0.7$). Across all methods, empirical coverage remains close to the nominal level $1-\alpha=0.99$.

\begin{table}[H]
\centering
\caption{
W-CROC results on MNIST (hflip+blur $\sigma{=}1.5$, $D{=}5$ streams, $n{=}400$, $\alpha{=}0.01$, $100$ split permutations, $200$ test tasks).
W-CROC (Hard) and MW-CROC (EDL-F) use $\beta{=}\varepsilon$;
$\varepsilon{=}0$ is filled with CROC.
MW-CROC (EDL-M) is trained on mixed contamination levels with $\beta{=}0.5$.
}
\label{tab:wcroc_mnist}
\footnotesize
\begin{tabular}{cl|cc|ccccc}
\toprule
$\varepsilon$ & Method & $|\mathcal{K}_{\alpha}(\boldsymbol{X})|_{\mathrm{pen}}$ & Cov.
  & $p_{(1)}$ & $p_{(2)}$ & $p_{(3)}$ & $p_{(4)}$ & $p_{(5)}$ \\
\midrule
\multirow{5}{*}{$0$}
  & CROC             & 1.000 & 1.000 & 0.990 & 0.010 & 0.010 & 0.010 & 0.010 \\
  & Oracle W-CROC    & 1.000 & 1.000 & 0.990 & 0.010 & 0.010 & 0.010 & 0.010 \\
  & W-CROC (Hard)    & 1.000 & 1.000 & 0.949 & 0.010 & 0.010 & 0.010 & 0.010 \\
  & MW-CROC (EDL-F)  & 1.000 & 1.000 & 0.949 & 0.010 & 0.010 & 0.010 & 0.010 \\
  & MW-CROC (EDL-M)  & 1.025 & 1.000 & 0.943 & 0.010 & 0.010 & 0.010 & 0.010 \\
\midrule
\multirow{5}{*}{$0.1$}
  & CROC             & 1.015 & 0.995 & 0.662 & 0.010 & 0.010 & 0.010 & 0.010 \\
  & Oracle W-CROC    & 1.015 & 1.000 & 0.995 & 0.010 & 0.010 & 0.010 & 0.010 \\
  & W-CROC (Hard)    & 1.480 & 1.000 & 0.866 & 0.010 & 0.010 & 0.034 & 0.041 \\
  & MW-CROC (EDL-F)  & 1.085 & 0.990 & 0.696 & 0.010 & 0.010 & 0.010 & 0.011 \\
  & MW-CROC (EDL-M)  & 1.040 & 1.000 & 0.820 & 0.010 & 0.010 & 0.010 & 0.010 \\
\midrule
\multirow{5}{*}{$0.3$}
  & CROC             & 1.675 & 0.995 & 0.519 & 0.017 & 0.030 & 0.015 & 0.014 \\
  & Oracle W-CROC    & 1.245 & 1.000 & 1.000 & 0.011 & 0.031 & 0.026 & 0.015 \\
  & W-CROC (Hard)    & 2.425 & 0.995 & 0.788 & 0.012 & 0.035 & 0.156 & 0.151 \\
  & MW-CROC (EDL-F)  & 1.630 & 0.990 & 0.520 & 0.016 & 0.023 & 0.020 & 0.019 \\
  & MW-CROC (EDL-M)  & 1.430 & 0.985 & 0.630 & 0.016 & 0.013 & 0.012 & 0.017 \\
\midrule
\multirow{5}{*}{$0.5$}
  & CROC             & 3.785 & 0.990 & 0.498 & 0.081 & 0.131 & 0.064 & 0.060 \\
  & Oracle W-CROC    & 1.865 & 1.000 & 0.995 & 0.072 & 0.081 & 0.091 & 0.086 \\
  & W-CROC (Hard)    & 3.545 & 0.990 & 0.653 & 0.062 & 0.124 & 0.324 & 0.323 \\
  & MW-CROC (EDL-F)  & 2.280 & 1.000 & 0.491 & 0.024 & 0.070 & 0.049 & 0.038 \\
  & MW-CROC (EDL-M)  & 1.985 & 0.995 & 0.532 & 0.047 & 0.056 & 0.023 & 0.030 \\
\midrule
\multirow{5}{*}{$0.7$}
  & CROC             & 4.975 & 1.000 & 0.526 & 0.352 & 0.412 & 0.327 & 0.314 \\
  & Oracle W-CROC    & 3.305 & 1.000 & 0.995 & 0.264 & 0.231 & 0.234 & 0.212 \\
  & W-CROC (Hard)    & 4.505 & 0.990 & 0.517 & 0.161 & 0.272 & 0.508 & 0.506 \\
  & MW-CROC (EDL-F)  & 3.170 & 0.995 & 0.474 & 0.116 & 0.121 & 0.084 & 0.077 \\
  & MW-CROC (EDL-M)  & 4.060 & 0.995 & 0.493 & 0.191 & 0.183 & 0.070 & 0.179 \\
\bottomrule
\end{tabular}
\end{table}

\subsection{Additional Coverage Results}

Table~\ref{tab:ablation_alpha} reports the effect of the miscoverage level $\alpha$ on MW-CONCH (EDL-M) across DomainNet, CIFAR-100, and MNIST. As expected, increasing $\alpha$ produces smaller confidence sets at the cost of lower empirical coverage. Across all datasets and contamination levels, the empirical coverage generally remains close to the nominal target level $1-\alpha$, supporting the finite-sample validity of the proposed conformal procedure. These results illustrate the expected trade-off between confidence set size and coverage controlled by the choice of $\alpha$.

\begin{table}[H]
\centering
\caption{
Effect of miscoverage level $\alpha$ on MW-CONCH (EDL-M).
Each cell reports the average confidence set size
$|\mathcal{C}_\alpha(\bX)|$
with empirical coverage in parentheses.
Theoretical guarantee: coverage $\ge 1-\alpha$.
}
\label{tab:ablation_alpha}
\small
\begin{tabular}{c|ccccc}
\toprule
 & \multicolumn{5}{c}{Contamination level $\varepsilon$} \\
\cmidrule{2-6}
$\alpha$ & $0.0$ & $0.1$ & $0.3$ & $0.5$ & $0.7$ \\
\midrule
\multicolumn{6}{l}{\textbf{DomainNet} } \\
\midrule
$0.05$ & 1.81\;(.955) & 1.95\;(.940) & 2.46\;(.960) & 3.57\;(.965) & 6.75\;(.965) \\
$0.10$ & 1.33\;(.900) & 1.41\;(.890) & 1.86\;(.880) & 2.75\;(.925) & 5.17\;(.930) \\
$0.15$ & 1.14\;(.855) & 1.21\;(.835) & 1.50\;(.800) & 2.32\;(.880) & 4.38\;(.905) \\
$0.20$ & 1.01\;(.830) & 1.05\;(.780) & 1.26\;(.755) & 1.91\;(.805) & 3.79\;(.865) \\
$0.25$ & 1.00\;(.830) & 1.00\;(.755) & 1.10\;(.735) & 1.65\;(.780) & 3.21\;(.805) \\
$0.30$ & 1.00\;(.830) & 1.00\;(.755) & 1.00\;(.710) & 1.32\;(.700) & 2.67\;(.745) \\
\midrule
\multicolumn{6}{l}{\textbf{CIFAR-100}} \\
\midrule
$0.05$ & 1.23\;(.945) & 1.39\;(.960) & 2.08\;(.935) & 3.22\;(.990) & 6.09\;(.925) \\
$0.10$ & 1.00\;(.905) & 1.14\;(.900) & 1.79\;(.910) & 2.67\;(.945) & 5.07\;(.895) \\
$0.15$ & 1.00\;(.905) & 1.03\;(.865) & 1.56\;(.855) & 2.28\;(.905) & 4.33\;(.845) \\
$0.20$ & 1.00\;(.905) & 1.00\;(.860) & 1.40\;(.795) & 2.02\;(.880) & 3.84\;(.785) \\
$0.25$ & 1.00\;(.905) & 1.00\;(.860) & 1.17\;(.735) & 1.78\;(.800) & 3.37\;(.735) \\
$0.30$ & 1.00\;(.905) & 1.00\;(.860) & 1.08\;(.710) & 1.54\;(.735) & 2.96\;(.700) \\
\midrule
\multicolumn{6}{l}{\textbf{MNIST}}  \\
\midrule
$0.05$ & 1.03\;(.970) & 1.08\;(.955) & 1.43\;(.920) & 1.60\;(.950) & 2.38\;(.965) \\
$0.10$ & 1.02\;(.960) & 1.02\;(.935) & 1.17\;(.860) & 1.34\;(.900) & 1.85\;(.865) \\
$0.15$ & 1.02\;(.960) & 1.02\;(.935) & 1.09\;(.815) & 1.22\;(.860) & 1.62\;(.780) \\
$0.20$ & 1.02\;(.960) & 1.02\;(.935) & 1.08\;(.810) & 1.17\;(.840) & 1.49\;(.745) \\
$0.25$ & 1.02\;(.960) & 1.02\;(.935) & 1.08\;(.805) & 1.14\;(.820) & 1.37\;(.695) \\
$0.30$ & 1.02\;(.960) & 1.02\;(.935) & 1.08\;(.805) & 1.14\;(.820) & 1.28\;(.655) \\
\bottomrule
\end{tabular}
\end{table}

\bibliographystyle{IEEEtran}
\bibliography{reference}

\end{document}